\documentclass[a4paper,fleqn]{cas-dc}

\usepackage[authoryear,longnamesfirst]{natbib}

\def\tsc#1{\csdef{#1}{\textsc{\lowercase{#1}}\xspace}}
\tsc{WGM}
\tsc{QE}

\newtheorem{proposition}{Proposition}
\newtheorem{definition}{Definition}
\newtheorem{corollary}{Corollary}
\newdefinition{rmk}{Remark}
\newproof{pf}{Proof}

\usepackage{listings}
\usepackage{xcolor}
\usepackage{multicol}

\definecolor{jsonString}{RGB}{42,0,255}      
\definecolor{jsonComment}{RGB}{34,139,34}    
\definecolor{jsonNum}{RGB}{170,9,111}        
\definecolor{jsonPunct}{RGB}{205,49,49}      
\definecolor{jsonDelim}{RGB}{20,105,176}     

\lstdefinelanguage{json}{
  basicstyle=\ttfamily\footnotesize,
  backgroundcolor=\color{gray!5},
  showstringspaces=false,
  breaklines=true,
  breakatwhitespace=true,
  frame=none,
  escapeinside={(*@}{@*)},
  stringstyle=\color{jsonString},
  morestring=[b]",
  commentstyle=\color{jsonComment}\itshape,
  morecomment=[l]{//},
  literate=
   *{0}{{{\color{jsonNum}0}}}{1}
    {1}{{{\color{jsonNum}1}}}{1}
    {2}{{{\color{jsonNum}2}}}{1}
    {3}{{{\color{jsonNum}3}}}{1}
    {4}{{{\color{jsonNum}4}}}{1}
    {5}{{{\color{jsonNum}5}}}{1}
    {6}{{{\color{jsonNum}6}}}{1}
    {7}{{{\color{jsonNum}7}}}{1}
    {8}{{{\color{jsonNum}8}}}{1}
    {9}{{{\color{jsonNum}9}}}{1}
    {:}{{{\color{jsonPunct}{:}}}}{1}
    {,}{{{\color{jsonPunct}{,}}}}{1}
    {\{}{{{\color{jsonDelim}{\{}}}}{1}
    {\}}{{{\color{jsonDelim}{\}}}}}{1}
    {[}{{{\color{jsonDelim}{[}}}}{1}
    {]}{{{\color{jsonDelim}{]}}}}{1},
}

\begin{document}
\let\WriteBookmarks\relax
\def\floatpagepagefraction{1}
\def\textpagefraction{.001}

\shorttitle{}    

\shortauthors{G. MARRAKCHI, B. MATEI.}  

\title [mode = title]{The Large Cancer Assistant (LCA): A Model-Agnostic Orchestration Framework for Scalable Clinical Decision Support in Oncology}  



%
\author[1]{Ghassen MARRAKCHI}[type=editor,
      orcid=0009-0006-2538-783X]
\cormark[1]
\ead{ghassen.marrakchi@lipn.univ-paris13.fr}
\ead[url]{https://marrakchighassen.github.io/}
\credit{Conceptualization, Investigation, Methodology, Project administration, Resources, Software, Visualization, Formal analysis – original draft, Writing – original draft, Writing – review and editing.}

\author[1]{Basarab MATEI}[type=editor,
      orcid=0000-0001-7946-530X]
\ead{matei@lipn.univ-paris13.fr}
\ead[url]{https://lipn.univ-paris13.fr/~matei/}
\credit{Supervision, Validation, Funding acquisition, Formal analysis – review and editing, Writing – review and editing.}

\affiliation[1]{organization={LIPN, CNRS, UMR 7030 --- Université Sorbonne Paris Nord},
            addressline={99 Av. Jean Baptiste Clément},
            city={Villetaneuse},
            postcode={F-93430},
            country={France}}

\cortext[1]{Corresponding author}


\begin{abstract}
\mbox{}

\textbf{Objective:} Multimodal deep learning models in oncology are currently limited by monolithic designs that rigidly couple data ingestion, clinical routing, and artificial intelligence (AI) inference. To address this inflexibility, we propose the Large Cancer Assistant (LCA), a model-agnostic, post-hoc orchestration framework designed for scalable clinical decision support.  

\textbf{Methods:} The LCA is mathematically formalized as a 7-tuple architecture grounded in the principle of Algorithmic Impermeability, ensuring the orchestration logic remains strictly independent of underlying black-box AI models. We introduce the Entry Theory, leveraging Geometric Deep Learning (GDL) to standardize multimodal patient data along distinct structural and medical axes. The system dynamically orchestrates data via a Cancer Switching Module and intentionally isolates the core AI execution from volatile hospital IT infrastructures by outputting a Standardized Intermediate Payload (SIP).  

\textbf{Results:} A Proof of Concept (PoC) validated the orchestration logic across four technical scenarios. The framework executed a nominal flow with negligible orchestration overhead. It empirically demonstrated algorithmic impermeability by maintaining an invariant routing projection during AI model swaps, and it validated strict failure-safety by achieving a 100\% recall rate in generating targeted Supplementary Data Requests (SDR) under injected data anomalies. Multi-protocol execution capability was also successfully verified.

\textbf{Conclusion:} By structurally decoupling multimodal ingestion from feature inference, the LCA provides a highly adaptable and modular orchestration foundation. The SIP establishes a clear architectural boundary, natively setting the stage for downstream Electronic Medical Record (EMR) interoperability as an independent future paradigm.  
\end{abstract}


\begin{highlights}
\item LCA orchestrates multimodal oncology data without model coupling.
\item Algorithmic impermeability strictly isolates AI from data routing.
\item Entry Theory algebraically standardizes diverse clinical formats.
\item A Standardized Intermediate Payload (SIP) ensures EMR isolation.
\item Proof of concept proves 100\% failure safety and zero AI overhead.
\end{highlights}



\maketitle


\section{Introduction}
\label{sec:intro}
The landscape of clinical oncology is inherently multimodal. A comprehensive patient diagnosis rarely relies on a single data source; rather, it requires the continuous synthesis of high-dimensional spatial imaging (e.g., computed tomography, magnetic resonance imaging), unstructured semantic histories (e.g., clinical notes, pathology reports), and tabular biological metrics (e.g., blood panels). Consequently, the integration of Artificial Intelligence (AI) into oncology has largely focused on developing multimodal deep learning models capable of processing these diverse data streams to improve diagnostic and prognostic accuracy.  

However, the clinical translation of these advanced models is currently constrained by systemic architectural flaws. Many existing multimodal systems remain tightly coupled to specific modality combinations and task definitions, limiting their flexibility across heterogeneous clinical workflows. They tightly couple the ingestion of data, the routing logic, and the core neural network inference into a single, rigid, monolithic pipeline. This monolithic design prevents dynamic adaptability: a model trained specifically for lung nodule segmentation cannot easily re-route an abdominal profile, nor can it dynamically adjust to missing data modalities without catastrophic failure. Furthermore, previous paradigms often conceptualized these systems for synchronous or concurrent execution. However, clinical realities demand a post-hoc, offline orchestration approach capable of asynchronously synthesizing highly fragmented patient histories without forcing artificial real-time constraints.  

Within the domain of biomedical informatics, the focus must fundamentally shift from the internal mechanics of individual predictive models to the overarching systems engineering that governs their clinical deployment. This paper explicitly addresses this gap by framing the challenge not as a computer vision or signal processing problem, but as an advanced informatics orchestration problem. By explicitly treating the underlying diagnostic AI networks as impermeable "black boxes," we isolate the computational inference from the clinical data flow. This perspective is vital for modern Clinical Decision Support Systems (CDSS), where the primary technical bottleneck is managing data heterogeneity and modular routing, rather than simply optimizing isolated algorithmic performance.

To address these structural limitations, this work presents a framework specification for the Large Cancer Assistant (LCA). Rather than proposing a singular diagnostic neural network, the LCA is conceptualized as an advanced systems-engineering orchestration architecture. It is designed to autonomously ingest, route, and synchronize multi-format patient data across specialized pathological pipelines without rigid coupling to the underlying computational models. A foundational pillar of the LCA is the principle of Algorithmic Impermeability. By establishing strict structural boundaries between data preprocessing, clinical routing, and AI inference, the framework ensures that its orchestration logic remains entirely independent of specific, transient machine-learning models.  

To formally establish these architectural boundaries, the main contributions of this work are defined as follows:
\begin{itemize}
    \item \textbf{The LCA Framework and Algorithmic Impermeability:} We formally define the LCA as a 7-tuple orchestration architecture that rigorously enforces Algorithmic Impermeability, mathematically guaranteeing that system-level orchestration projections remain invariant under AI model substitutions.
    \item \textbf{Input Data Formalization (Entry Theory):} We introduce a novel Entry Theory leveraging Geometric Deep Learning (GDL) to systematically standardize multimodal clinical inputs across distinct structural and medical axes, ensuring the mathematical independence of data topology from clinical provenance.
    \item \textbf{Dual-Variant Pathology Routing:} We define a modular Cancer Switching Module capable of dynamic pipeline activation via either explicit deterministic parameterization (V1) or an autonomous, calibrated multi-protocol learned router (V2).
    \item \textbf{The Standardized Intermediate Payload (SIP):} We establish a deliberate architectural boundary by culminating the system's output in a decoupled, proprietary JSON payload (SIP) accompanied by a structured Supplementary Data Request (SDR) feedback loop. This explicitly isolates the core AI framework from downstream Electronic Medical Record (EMR) interoperability dependencies.  
\end{itemize}

The remainder of this paper is organized as follows: Section \ref{sec:related} reviews the related literature regarding multimodal oncology and clinical decision support. Section \ref{sec:method} details the methodological foundations, including the formal Entry Theory, the LCA framework, its internal modules, and the SIP specification boundary. Section \ref{sec:poc} presents the empirical Proof of Concept (PoC) results, demonstrating the framework's nominal flow, algorithmic impermeability, and failure safety. Section \ref{sec:disc} outlines the principal findings, framework limitations, and ethical considerations as a CDSS. Finally, Section \ref{sec:conl} concludes the paper with a defined roadmap for future interoperability translation.  

\begin{figure*}[tbp]
    \centering
    \includegraphics[width=.9\textwidth]{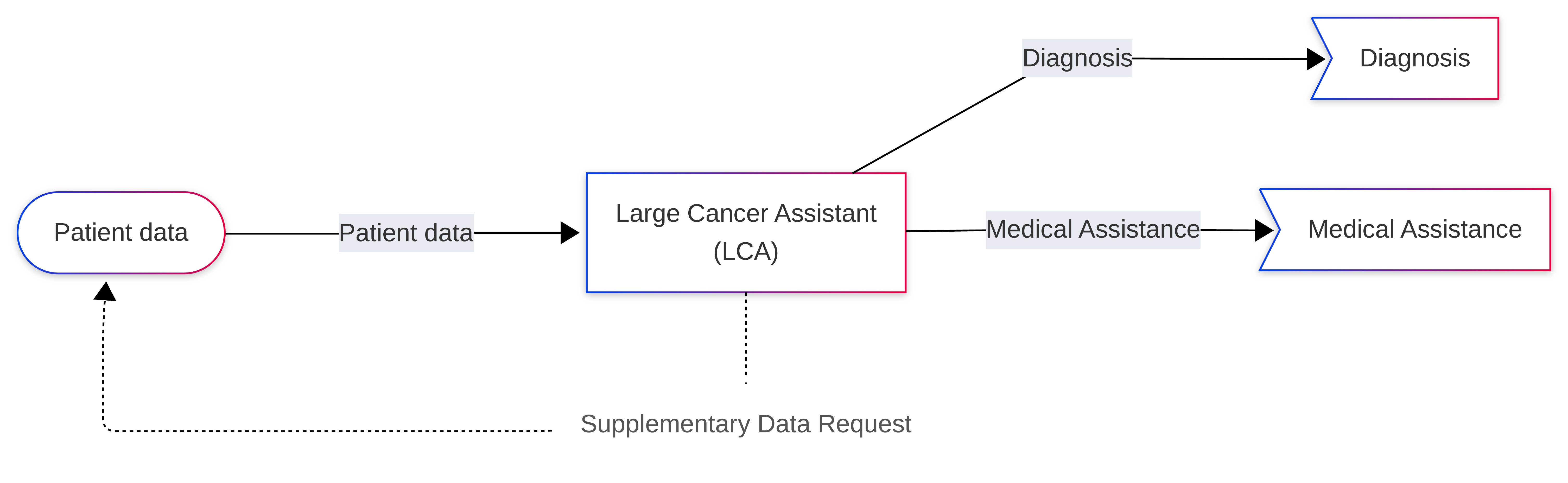}
    \caption{Macroscopic ecosystem of the LCA framework, showing the central black box "LCA" with an incoming "Multimodal patient data" stream and three distinct outputs: Diagnosis, Medical Assistance, and a Supplementary Data Request feedback loop.}
    \label{fig:macroscopic_ecosystem}
\end{figure*}

\section{Related Work}
\label{sec:related}

This section surveys prior work relevant to the Large Cancer Assistant (LCA) framework across two primary axes: multimodal deep learning in oncology and orchestration-style clinical decision support systems (CDSS). Together, these strands demonstrate that while powerful AI models exist in isolation, the combination of strict algorithmic impermeability and modular, multi-pathology switching remains underexplored.

\subsection{Multimodal Deep Learning in Oncology}
\label{subsec:related-multi_modal_dl}

A substantial literature has emerged focusing on the joint modeling of clinical modalities to improve diagnostic and prognostic accuracy. Early frameworks established the value of combining multi-omics inputs (mRNA, miRNA, DNA methylation) using deep learning and graph-based feature selection to outperform single-modality baselines \cite{GD_Net}. Broader multi-omics models leverage deep neural networks to fuse genomics and epigenetics with clinical covariates for survival analysis and risk stratification \cite{DeepOmix}.

Recent advancements have extended this integration to include spatial imaging and unstructured clinical variables. For example, systems have been developed to integrate pre-treatment histopathology images, molecular biomarkers, and clinical data using CNN-based image encoders alongside dense layers for tabular features, demonstrating the necessity of explicit cross-modal representation learning \cite{DeepClinMed_PGM}. Similarly, multimodal state-space models combine endoscopic images, radiomic features, and clinical variables within unified probabilistic frameworks to capture temporal disease trajectories \cite{DeepMultimodal_Colorectal}. At the scale of foundation models, large-scale multimodal transformers have been trained on millions of pathology images and free-text clinical reports to predict prognosis and immunotherapy response across multiple tumor types \cite{MUSK_Stanford}.

However, a persistent limitation across these architectures is their monolithic design \cite{Multimodal_Review_2024}. These systems typically operate as rigidly coupled prediction models, tightly bound to their training cohorts and specific data modality combinations \cite{Precision_Oncology_Survey}. Because their data orchestration is permanently fused to their underlying feature extractors, they prevent modular updates or adaptable clinical routing. Consequently, the field currently lacks a mechanism to structurally decouple multimodal data ingestion from deep learning inference---a critical void that the LCA resolves through its foundational principle of Algorithmic Impermeability.

\subsection{Multimodal Clinical Decision Support Systems}
\label{subsec:related-multi_modal_cad}

Parallel to model-centric advancements, several systems target end-to-end decision support, embedding AI directly into clinical workflows. Task-focused CDSS architectures frequently fuse radiological imaging, dosimetric parameters, and clinical covariates to support personalized radiotherapy \cite{Multimodal_Radiotherapy_NSCLC} or predict post-operative recurrence \cite{DeepMultimodal_Colorectal}. Beyond these task‑specific tools, several CDSS architectures have been implemented as workflow‑engine–driven systems that encode clinical guidelines via integrated workflow and rule engines, yet they remain tightly bound to local institutional pathways and lack a cancer‑agnostic orchestration layer \cite{Lee2010, Huser2011}.

While these systems demonstrate the clinical necessity of AI empowerment, their logic remains closely tied to specific institutional contexts or isolated anatomical targets (e.g., exclusively addressing non-small-cell lung cancer or colorectal tumors). Because they lack a universal, cancer-agnostic ingestion framework, translating these siloed CDSS tools to generalized hospital environments is fundamentally restricted. They require an overarching, model-agnostic orchestration layer---which the LCA introduces via its dynamic Cancer Switching Module--- to autonomously align highly diverse patient profiles with the appropriate specialized clinical protocols.

\subsection{Positioning of the Large Cancer Assistant (LCA)}
\label{subsec:related-position}

The proposed LCA framework deliberately departs from existing multimodal oncology paradigms along three strategic dimensions:

\begin{itemize}
\item \textbf{Algorithmic Impermeability vs. Monolithic Fusion:} Unlike existing multimodal transformers or deep-fusion networks that directly consume diverse inputs to generate a unified probability score \cite{GD_Net, MUSK_Stanford}, the LCA enforces strict modular boundaries. It separates multimodal data preprocessing from the core diagnostic inference, ensuring the orchestration logic remains fully independent of the underlying neural networks.

\item \textbf{Dynamic Switching vs. Task-Specific CDSS:} Rather than operating as a siloed tool for a single cancer type \cite{Multimodal_Radiotherapy_NSCLC}, the LCA utilizes a distinct \emph{Cancer Switching Module}. This allows the framework to act as a universal orchestrator, dynamically detecting the anatomical target and activating the appropriate, specialized downstream pathological pipeline.

\item \textbf{Standardized Output Boundary vs. Direct EMR Integration:} Unlike systems that directly commit results to hospital databases, the LCA produces a Standardized Intermediate Payload (SIP) as its sole external output. This boundary strictly decouples the AI orchestration core from volatile EMR infrastructure: changes to HL7 FHIR versions, database schemas, or institutional IT systems require no modification of any upstream LCA module. Downstream FHIR translation is scoped as an independent subsequent work. This design is consistent with interoperability‑driven CDSS architectures that leverage standards such as HL7 FHIR, while deliberately keeping the orchestration layer model‑agnostic and EMR‑decoupled \cite{carbonaro2025raw, Jung2022}.
\end{itemize}

\section{Overview \& Design Principle}
\label{sec:method}

A critical limitation of existing automated diagnostic tools is their inability to process the full spectrum of patient data, often isolating radiological imaging from the broader clinical context. To address this, the Large Cancer Assistant (LCA) orchestration framework is explicitly designed to encapsulate the entire decision assistance schema in a post-hoc (offline) manner. It ingests and aggregates a multimodal triad of patient data: high-dimensional spatial imaging (e.g., CT, MRI), unstructured textual data (e.g., clinical notes), tabular biological metrics (e.g., blood biomarkers), etc., after acquisition is complete. The LCA does not operate concurrently with an ongoing examination (e.g., it does not process a live ultrasound stream during the exam); no real-time processing is intended.

A foundational design principle of the LCA is Algorithmic Impermeability. The framework imposes strict modular boundaries between its clinical routing logic and the underlying neural networks. By decoupling data orchestration from feature inference, the LCA ensures that its internal logic remains invariant under machine-learning model swaps, provided the interface contracts are satisfied. This impermeability ensures that the system is highly scalable, capable of integrating future interoperability standards natively, and strictly independent of specific, transient machine-learning models. The internal orchestration is structured as a unidirectional pipeline governed by a directed acyclic graph (DAG), propagating a shared routing context denoted as $\hat{P}$ through the successive modules (See Figure \ref{fig:sec:method}).

\begin{figure*}[tbp]
    \centering
    \includegraphics[width=\textwidth]{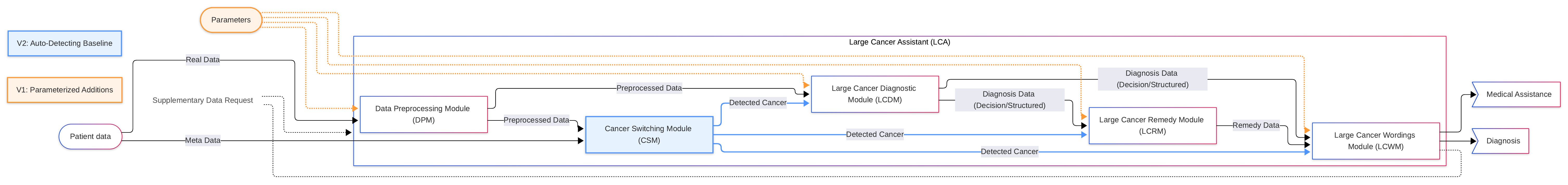}
    \caption{Internal orchestration graph (unidirectional DAG) showing the pipeline sequence: DPM $\rightarrow$ CSM $\rightarrow$ \{LCDM, LCRM\} $\rightarrow$ LCWM $\rightarrow$ SIP. The CSM includes V1 (parameterized) / V2 (auto-detection) annotations. An SDR branch originates from the LCWM. The shared routing context $\hat{P}$ is propagated throughout the modules.}
    \label{fig:sec:method}
\end{figure*}

\subsection{Input Data Formalization (Entry Theory)}
\label{subsec:method-entry}
To orchestrate heterogeneous clinical data comprehensively, the LCA framework employs a formal Entry Theory (depicted in Figure \ref{fig:subsubsec:method-entry-charac}) that decouples the structural properties of an input from its clinical semantics.

\subsubsection{CS Axis (Structural)}
\label{subsubsec:method-entry-cs}
Following the principles of Geometric Deep Learning, input data is formalized structurally. A Domain (Definition 1) is defined as a pair $(\Omega, \mathfrak{G})$, where $\Omega$ is a set of positions and $\mathfrak{G}$ is a symmetry group acting on $\Omega$. An Entry (Definition 2) is a geometric signal $x \in \mathcal{C}^{\Omega}$ defined on this domain, with $\mathcal{C}$ being the feature space. This definition elegantly unifies spatial modalities (e.g., CT volumes, Whole-Slide Images) and non-spatial ones (e.g., textual clinical notes, tabular biological metrics) under a common mathematical representation

\subsubsection{Medical Axis}
\label{subsubsec:method-entry-med}
Independently of its structure, each entry is augmented with a clinical context, defined by three categorical attributes:
\begin{itemize}
    \item \textbf{Provenance ($\rho$):} The acquisition process (e.g., imaging/CT, pathology/WSI).
    \item \textbf{Usage ($u$):} The functional role in the decision process (e.g., observation, diagnosis).
    \item \textbf{Epistemic Certainty ($\kappa$):} The evidential grounding of the information (e.g., confirmed, suspected, inferred).
\end{itemize}
Together, these form the Medical Signature $\sigma = (\rho, u, \kappa)$ (Definitions \ref{def:entry-prov}-\ref{def:entry-med_sig}), capturing the clinical essence of the data.

\subsubsection{Characterized Entry \& Independence}
\label{subsubsec:method-entry-charac}
A Characterized Entry $\tilde{x} = (x, \Omega, \mathfrak{G}, \mathcal{C}, \sigma)$ (Definition \ref{def:entry-char_entry}) merges these two axes. Crucially, the CS structural axis and the medical axis are mutually independent (Proposition \ref{prop:entry-ind_axe}). The geometric structure of an entry does not dictate its medical signature, and vice versa. For instance, two structurally identical CT volumes can have different certainty levels depending on their conclusiveness, while two entries with identical signatures (e.g., biological observations) might have entirely different tabular geometries.

Finally, an LCA Input (Definitions \ref{def:entry-clinic_hist}-\ref{def:entry-lca_input}) is formalized as a clinical history: a finite, ordered sequence of characterized entries $(\tilde{x}_j)_{j=1}^{T}$ associated with a patient identifier, fully supporting multi-modal heterogeneity and arbitrarily short histories ($T \ge 1$); the degenerate case $T = 1$ -- a single new-patient visit or single-modality input -- requires no separate treatment.(see \ref{app:entry} for detailed entry theory explanation).

\begin{figure*}[tbp]
    \centering
    \includegraphics[width=\textwidth]{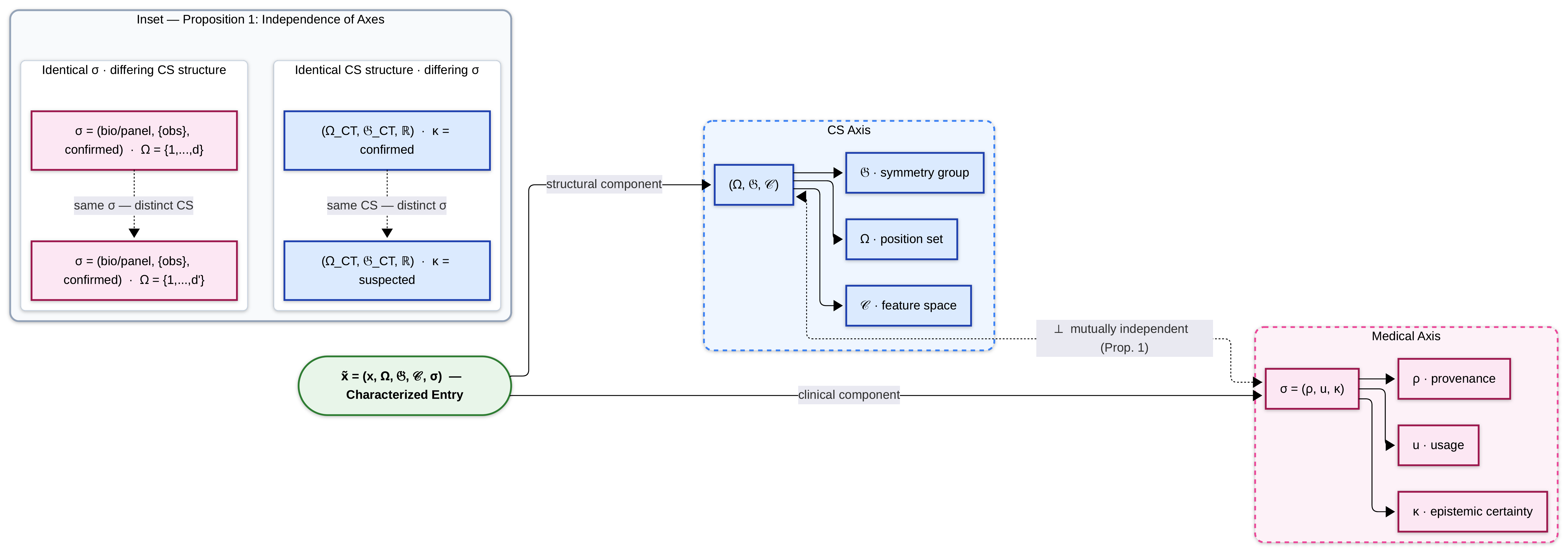}
    \caption{Structure of the characterized entry and orthogonality of axes, decomposing an entry into the CS axis ($\Omega, \mathfrak{G}, \mathcal{C}$) and the medical axis $\sigma=(\rho,u,\kappa)$. The inset illustrates Proposition \ref{prop:entry-ind_axe} (identical CS structure with differing signatures, and identical signatures with differing CS structures).}
    \label{fig:subsubsec:method-entry-charac}
\end{figure*}

\subsection{The LCA Framework}
\label{subsec:method-lca}
Formally, the Large Cancer Assistant (LCA) framework is structured as a 7-tuple. This architecture comprises the input space ($\mathcal{E}$), a shared protocol catalog $\mathcal{P} = \{p_1, \ldots, p_K\}$ (a finite, non-empty set of $k$ clinical protocols ; Definition \ref{def:csm-csm}), and a unidirectional orchestration pipeline consisting of five specific modules: the Data Preprocessing Module ($f_{DP}$), the Cancer Switching Module ($f_{CS}$), the Large Cancer Diagnostic Module ($f_{LCD}$), the Large Cancer Remedy Module ($f_{LCR}$), and the Large Cancer Wordings Module ($f_{LCW}$) ; where the global LCA function is depicted as : 
\begin{equation}
\text{LCA} \;=\; \bigl(\mathcal{E},\;\mathcal{P},\;f_{DP},\;f_{CS},\;f_{LCD},\; f_{LCR},\;f_{LCW}\bigr)
\end{equation}

\subsubsection{System-Level Algorithmic Impermeability}
\label{subsubsec:method-lca-impermeability}
The core guarantee of this mathematical formulation is defined by Proposition \ref{prop:lca-impermeab} (System Impermeability). It states that if two different machine-learning models satisfy the same diagnostic or therapeutic interface contract, swapping them will leave the orchestration-structural projection of the system entirely invariant. Specifically, the module activation sequence, the routing decisions, the failure handling mechanisms, and the final structural schema of the Standardized Intermediate Payload (SIP) remain identical. The clinical inference content will naturally differ between the two models, but the orchestration layer remains impermeable to this change. As a direct corollary (Corollary \ref{cor:lca-unchange}), downstream modules like the Data Preprocessing Module (DPM), Cancer Switching Module (CSM), and Large Cancer Wordings Module (LCWM) require no updates when underlying AI models are retrained or replaced.

\subsubsection{Key Framework Properties}
\label{subsubsec:method-lca-key_pptes}

This formalization guarantees four essential properties across the entire framework:
\begin{enumerate}[\bf (Q1)]
    \item \textbf{Defined on $\mathcal{E}$:} The framework processes all valid inputs, including partial clinical histories, without silent failures. More details in \ref{app:subsec:lca-q1_detailed}.
    \item \textbf{Protocol generality.} The framework is not specialized to any cancer type $c^*$. The protocol catalog $\mathcal{P}$ is a parameter; adding or removing a protocol $p_k$ requires updating the catalog and the corresponding module instances, not the orchestration.
    \item \textbf{Unidirectionality.} The pipeline graph is a DAG with a single source ($e$) and a single sink (SIP). There are no feedback loops and no bidirectional dependencies between modules.
    \item \textbf{Failure safety.} Every module that fails to process its input emits $\bot$ rather than a spurious output. The LCWM translates each $\bot$ into a Supplementary Data Request. No failure is silent.
\end{enumerate}


\subsubsection{Data Preprocessing Module (DPM)}
\label{subsubsec:method-ordchs-dpm}
The Data Preprocessing Module ($f_{DP}$) maps raw heterogeneous clinical inputs to canonical, standardized representations. Structural preprocessing is parameterized exclusively by the provenance $\rho$ of each entry, which selects the appropriate canonical form; the clinical attributes $u$ and $\kappa$ are not accessed and pass through unmodified. The full medical signature $\sigma = (\rho, u, \kappa)$ is never modified by the DPM. Formally, $f_{DP}$ is the sequential composition of five stages $f_0,\ldots,f_4$: history sequencing, signature extraction and preprocessing routing, structural preprocessing via provenance-specific operators $\phi_\rho$, axis reconstruction, and history recomposition. The associated properties---symmetry group preservation (P1.1), medical transparency (P1.2), metadata invariance (P1.3), and identity as degenerate case (P1.4)---and the complete formal definitions are detailed in \ref{app:dpm}.

\subsubsection{Cancer Switching Module (CSM)}
\label{subsubsec:method-ordchs-csm}
The Cancer Switching Module ($f_{CS}$) is the central routing engine of the LCA framework. Rather than enforcing a single definitive cancer diagnosis, the CSM receives the preprocessed clinical history from the DPM and outputs an activation set denoted $\hat{P}$. This set identifies the specific cancer protocols (from the a priori catalog $\mathcal{P}$ ; Definition \ref{def:csm-csm}) that warrant downstream investigation.

The framework supports two distinct routing variants:
\begin{itemize}
    \item \textbf{Variant V1 (Deterministic Activation):} Activation is governed by an explicit, a priori declaration parameter ($P_\lambda$). V1 bypasses the content of the clinical history entirely, routing the data strictly according to prior clinical or systemic instruction.
    \item \textbf{Variant V2 (Probabilistic Auto-Detection):} A learn\-ed, multi-protocol routing function that evaluates both the CS structural axis and the medical axis to infer clinical relevance. V2 activates protocols independently based on a calibrated confidence score exceeding a declared threshold $\tau \in (0,1)$. (Note: While fully specified theoretically within the framework, V2 is excluded from the current Proof of Concept).
\end{itemize}

Crucially, if evidence is insufficient in V2, or if the parameter is missing in V1, the CSM emits a null token ($\bot$), safely terminating the pipeline before inference and triggering a top-level Supplementary Data Request. (See \ref{app:csm} for details)

\subsubsection{Abstract AI Module}
\label{subsubsec:method-ordchs-ai_abstract}
The Large Cancer Diagnostic Module (LCDM) and the Large Cancer Remedy Module (LCRM) share an identical foundational architecture defined by the Abstract AI Module. This abstraction enforces a strict interface contract ($\mathcal{I}_\tau^{(k)}$) that mandates predefined input types, failure precondition sets, and explicit output space typologies. Output types are categorized globally as either structural (geometric signals defined on an output domain) or decisional (probability distributions over a class catalog). To guarantee algorithmic impermeability at the module level, a crucial distinction is made regarding failure handling: if an input lacks required modalities or fails predefined quality checks, the module emits a null token ($\bot$) to halt the pipeline securely. Conversely, if the underlying machine learning model fails internally during inference on valid data, the module must output a declared $\mathtt{CODE\_FAIL}$ rather than $\bot$, preserving the integrity of the framework's failure-safety mechanics.

\subsubsection{Large Cancer Diagnostic and Remedy Modules (LCDM \& LCRM)}
\label{subsubsec:method-ordchs-lcdm_lcrm}
The LCDM and LCRM instantiate the abstract interface to perform active cancer diagnostics and treatment generation, respectively. Executing strictly downstream of the CSM, the LCDM consumes the preprocessed clinical history and the activation set, producing novel characterized diagnostic entries. During this operation, the medical axis of the historical data is not discarded; instead, its epistemic certainty is renewed and synthesized into the output's medical signature.

Operating subsequently, the LCRM receives the complete diagnostic outputs alongside the untouched historical data. Both modules ensure per-protocol independence, operating on parallel isolated instances for each activated protocol. In specific clinical configurations where a protocol requires only a remedy phase (e.g., when the diagnosis is already definitively confirmed in the patient's history), the framework executes a canonical lifting operation. This mechanism bypasses the LCDM and maps the historically confirmed diagnosis directly into the requisite input format for the LCRM, strictly preserving the unidirectional flow of the pipeline (See \ref{app:lcdm_lcrm} for details).

\subsubsection{Large Cancer Wordings Module (LCWM)}
\label{subsubsec:method-ordchs-lcwm}
The Large Cancer Wordings Module ($f_{LCW}$) functions as the terminal node of the unidirectional pipeline and constitutes the exclusive external interface of the LCA framework. Like the DPM and the CSM, $f_{LCW}$ is $\theta$-free (Corollary~\ref{cor:lca-unchange}): it is not subject to the interface-contract substitution of Definition~\ref{def:lcdm_lcrm-mod_inter}, which applies only to $\tau \in \{\text{LCD}, \text{LCR}\}$. This fixed module aggregates the activation set, the structural and decisional outputs from the LCDM and LCRM, and the bypassed preprocessed clinical history.

From these inputs, the LCWM produces a unified Standardized Intermediate Payload (SIP). The SIP acts as a strict architectural boundary between the internal orchestration logic and external healthcare IT infrastructure. It is composed of two distinct elements:
\begin{itemize}
    \item \textbf{Text output:} A natural language narrative generated by an internal Natural Language Generation (NLG) model.
    \item \textbf{Structured bypass:} The unmodified structured AI outputs and historical data, passed through identically for downstream clinical archiving and auditing.
\end{itemize}

A critical feature of the LCWM is its enforcement of failure-safety via the Supplementary Data Request (SDR) track. If any upstream module emits a null token ($\bot$), the LCWM translates this precise failure into a targeted SDR. These requests are either top-level (when the CSM fails to activate any protocol) or protocol-level (when a specific LCDM or LCRM instance encounters a precondition failure). This mechanism guarantees that missing clinical evidence systematically halts the pipeline for the affected protocol and triggers human-in-the-loop intervention, precluding silent failures or uninformative outputs (See \ref{app:lcwm_sdr_sip} for details).

\subsection{Implementation Boundary: SIP/SDR Specification}
\label{subsec:method-sip_impl}
As the exclusive architectural boundary of the LCA framework, the Standardized Intermediate Payload (SIP) must strictly adhere to specific design principles to ensure interoperability and algorithmic impermeability (see Figure \ref{fig:subsec:method-sip_impl}). Several key principles explicitly govern its implementation:

\begin{enumerate}
    \item[\bf (D1)] \textbf{(SIP as the sole output type):} The LCWM invariably produces a SIP. Regardless of success or failure, the top-level payload structure remains a SIP. If a module failure occurs ($\bot$), the resulting Supplementary Data Request (SDR) is not a standalone payload type but is structurally embedded within the SIP's narrative and protocol outputs.
    \item[\bf (D2)] \textbf{($D_{out}$ by reference):} To prevent unacceptable latency and redundancy, the structured bypass does not embed heavy input content (e.g., full CT volumes or text tokens). Instead, the SIP logs input provenance via entry\_refs, mapping each entry to its acquisition index, its medical signature ($\sigma$), and its external identifier in the upstream healthcare system (e.g., FHIR, PACS).
    \item[\bf (D8)] \textbf{(Impermeability to serialization):} Machine-learning model identities (e.g., specific neural network versions like "monai\_unet\_v1\_lung") are strictly excluded from the SIP. Exposing such identities would create downstream dependencies, violating the algorithmic impermeability guarantee (Proposition 2). These details are reserved exclusively for internal audit logs.
    \item[\bf (D12)] \textbf{($\bot$ vs CODE\_FAIL distinction):} The SIP structure explicitly distinguishes between orchestration halting and inferential failure. The $\bot$ token, which halts the pipeline and triggers an SDR, is reserved solely for missing or invalid inputs (precondition failures, $\mathcal{D}_\bot^{(k)}$). Conversely, if an AI model fails internally during inference on a valid input, it outputs a protocol-specific CODE\_FAIL class rather than $\bot$, ensuring downstream consumers can precisely audit the failure's true origin without halting the orchestration logic erroneously.
\end{enumerate}

The complete set of design principles D1–D12, along with the full annotated JSON schema, are detailed in \ref{app:sip_sdr_design}. The comprehensive coverage of the eight canonical orchestration cases, demonstrating the SIP behavior under various module states, is provided in \ref{app:sip_sdr_cases}.

\begin{figure*}[tbp]
    \centering
    \includegraphics[width=\textwidth]{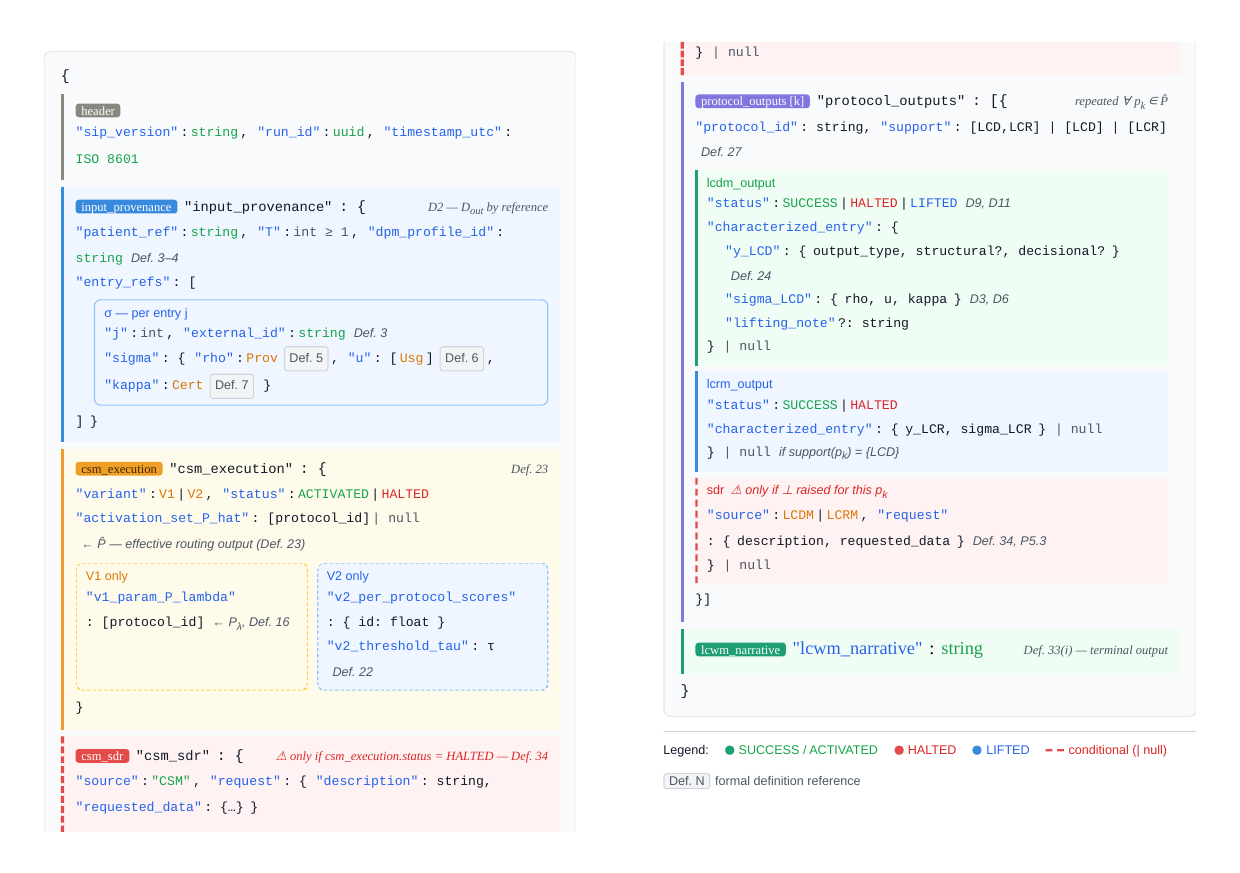}
    \caption{Condensed SIP/SDR schema (annotated pseudo-JSON representation) showing the block layout: header; input\_provenance (entry\_refs + $\sigma$); csm\_execution (variant, $\hat{P}$, V1/V2 settings); csm\_sdr; protocol\_outputs[$k$] containing individual lcdm\_output, lcrm\_output, or protocol-level sdr; and finally the lcwm\_narrative block.}
    \label{fig:subsec:method-sip_impl}
\end{figure*}

\begin{table*}[width=.9\textwidth, cols=6, pos=tbp]
    \caption{Coverage of Orchestration Situations (T1 Source)}
    \label{tab:subsec:method-sip_impl}
    \begin{tabular*}{\tblwidth}{@{} LLLLLL @{} }
        \toprule
        Case & CSM Variant & Support & LCDM & LCRM & SDR \\
        \midrule
        1 & V1 & \{LCD, LCR\} & SUCCESS & SUCCESS & --- \\
        2 & V2 multi-proto & \{LCD, LCR\} + \{LCD\} & SUCCESS & SUCCESS / null & --- \\
        3 & V2 & --- & not invoked & not invoked & CSM top-level \\
        4 & V1 & \{LCD, LCR\} & HALTED & null (not invoked) & LCDM level \\
        5 & V1 & \{LCD, LCR\} & SUCCESS & HALTED & LCRM level \\
        6 & V1 & \{LCR\} only & LIFTED & SUCCESS & --- \\
        7 & V1 multi-proto & \{LCD, LCR\} + \{LCD, LCR\} & SUCCESS / HALTED & SUCCESS / null & LCDM level (partial) \\
        8 & V1 & \{LCD, LCR\} & SUCCESS (CODE\_FAIL) & HALTED & LCRM level \\
        \bottomrule
    \end{tabular*}
\end{table*}

\section{Results \& Proof of Concept Setup}
\label{sec:poc}
The proof of concept (PoC) is governed by a single axiom: it evaluates the \emph{orchestration} behaviour of the framework--- routing, failure handling, and output structuring ---not the clinical performance of the underlying models. Accordingly, classifier-level metrics such as AUC or the Dice coefficient are deliberately excluded. All scenarios target the deterministic variant V1 on a single lung protocol; the diagnostic module is instantiated with interchangeable open reference bundles and the remedy module with a declared rule-based stub, as summarised in Fig.~\ref{fig:sec:poc}. Full environment, metric definitions, and sample sizes are given in~\ref{app:poc}.

\begin{figure*}[tbp]
    \centering
    \includegraphics[width=\textwidth]{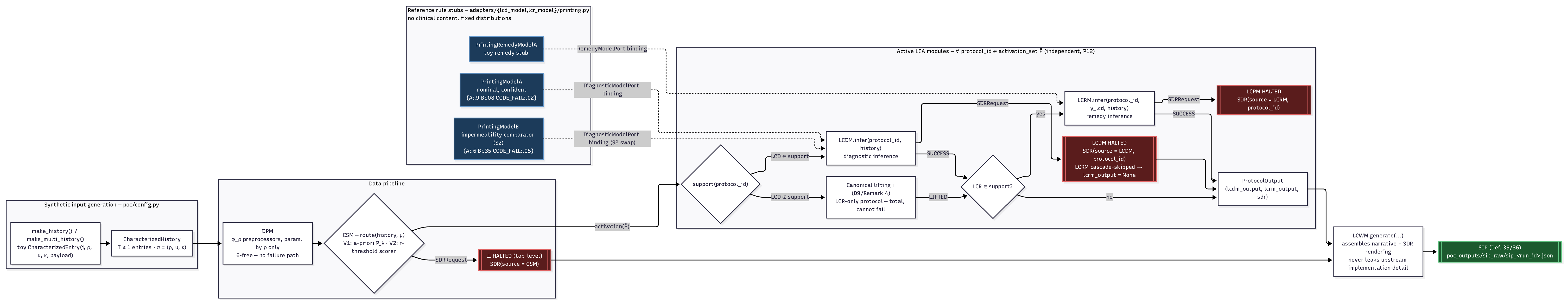}
    \caption{PoC architecture diagram showing the data pipeline, active LCA modules, synthetic input generation, and reference rule stubs.}
    \label{fig:sec:poc}
\end{figure*}

\subsection{Scenario 1: Nominal Flow}
\label{subsec:poc-s1}
Over $N=10$ runs the pipeline completed end-to-end without emitting $\bot$ (completion rate $100\%$). Every payload was schema-valid and carried the expected diagnostic signature $\sigma_{LCD}=(\textsf{inference/AI},\{\textsf{diagnosis}\}, \textsf{inferred})$ (Table~\ref{tab:poc-s1}). We do not claim the orchestration cost dominates inference: the reference components used here are lightweight. Rather, the orchestration layer itself (DPM, CSM, LCWM; $\approx 0.04$~ms) adds a sub-millisecond cost that is negligible relative to the $10^{2}$--$10^{3}$~ms typical of production 3D inference, confirming that orchestration introduces no material overhead in deployment.

\subsection{Scenario 2: Algorithmic Impermeability}
\label{subsec:poc-s2}
We substituted the diagnostic module's output stub (Stub A, emitting a 3D U-Net-style output schema) with an architecturally distinct stub (Stub B, emitting a SegResNet-style output schema), both satisfying the same interface contract $\mathcal{I}_{LCD}^{(\mathrm{lung})}$, on $N=10$ identical inputs. The orchestration-structural projection was invariant in every pair ($\pi$-equality $100\%$) while the diagnostic content differed in every pair ($100\%$; Table~\ref{tab:poc}). This is a direct empirical instantiation of Proposition~2: the routing projection is impermeable to the model swap even though the inference content is not.

\subsection{Scenario 3: Failure Safety}
\label{subsec:poc-s3}
We separate two structurally different guarantees (Table~\ref{tab:poc}).
\emph{Module-level failure safety} concerns inputs that reach a module and violate its declared precondition set $\mathcal{D}_{\bot}^{(k)}$. For both a corrupted CT and a missing required modality ($N=10$ each), the diagnostic module raised $\bot$ and the LCWM emitted a protocol-level SDR with the correct source ($\textsf{LCDM}$) and a non-generic request, yielding $100\%$ SDR recall---where recall is the conjunction of (i) $\bot$ raised, (ii) correct source, (iii) SDR emitted, and (iv) non-generic request. \emph{Type safety} is a distinct property: an empty history ($T=0$) lies outside the input space $\mathcal{E}$ (Def.~\ref{def:entry-lca_input}, $T\geq 1$) and is rejected at construction, before any module executes.
This produces no SDR---it is a violation of the type invariant, not a module failure---and was correctly rejected in $100\%$ of cases ($N=10$).

\subsection{Scenario 4: Multi-Protocol Execution ($K=2$)}
\label{subsec:poc-s4}
With $P_{\lambda}=\{p_{\mathrm{lung}},p_{\mathrm{mock}}\}$, both diagnostic instances executed as independent per-protocol branches in every run ($N=10$): branch independence, composite payload schema validity, and a protocol-output cardinality of exactly two each held at $100\%$ (Table~\ref{tab:poc}). This confirms that the framework composes multiple protocols without cross-contamination, addressing concerns specific to single-protocol ($K=1$) evaluation.

\begin{table}[tbp]
    \caption{S1 (nominal flow, $N=10$): per-module latency and payload validation.}
    \label{tab:poc-s1}
    \begin{tabular*}{\tblwidth}{@{} LC @{} }
        \toprule
        Module & Latency (mean $\pm$ std, ms) \\
        \midrule
        $t_{DPM}$            & $0.025 \pm 0.004$ \\
        $t_{CSM}$            & $0.009 \pm 0.004$ \\
        $t_{LCDM}$           & $0.032 \pm 0.007$ \\
        $t_{LCRM}$           & $0.022 \pm 0.003$ \\
        $t_{LCWM}$           & $0.008 \pm 0.004$ \\
        $t_{\mathrm{total}}$ & $0.096 \pm 0.016$ \\
        \midrule
        Completion rate         & $100\%$ \\
        SIP schema valid        & $100\%$ \\
        $\sigma_{LCD}$ correct  & $100\%$ \\
        \bottomrule
    \end{tabular*}
\end{table}

\begin{table}[tbp]
    \caption{Orchestration-property outcomes for S2--S4 ($N=10$ each; all targets $100\%$). The empty-history row is a type-system invariant (Def.~4), not a module-level failure, and emits no SDR.}
    \label{tab:poc}
    \begin{tabular*}{\tblwidth}{@{} LLC @{} }
        \toprule
        Sc. & Property & Rate \\
        \midrule
        S2 & $\pi$-equality (projection invariant)             & $100\%$ \\
        S2 & Content difference ($\tilde{y}_{LCD}$)            & $100\%$ \\
        \midrule
        S3 & SDR recall --- corrupted data$^{a}$               & $100\%$ \\
        S3 & SDR recall --- missing modality$^{a}$             & $100\%$ \\
        S3 & Type-invariant rejection --- empty history$^{b}$  & $100\%$ \\
        \midrule
        S4 & Branch independence                               & $100\%$ \\
        S4 & Composite SIP schema valid                        & $100\%$ \\
        S4 & Protocol-output cardinality $=2$                  & $100\%$ \\
        \bottomrule
    \end{tabular*}
    {\footnotesize 
    $^{a}$ Recall $=(\bot\ \text{raised})\wedge(\text{correct source})\wedge (\text{SDR emitted})\wedge(\neg\,\text{generic})$; source $=\textsf{LCDM}$ (protocol-level).\\
    $^{b}$ Rejected at construction ($T<1$); no SDR is produced.}
\end{table}

\section{Discussion}
\label{sec:disc}
The PoC validates the framework at the level it is designed to operate: orchestration, not classification. Across the four scenarios the system exhibited 
(i)~end-to-end completion with structurally valid payloads; 
(ii)~empirical algorithmic impermeability—the routing projection was invariant under a labeled stub substitution with differing simulated content, an experimental instantiation of Proposition~\ref{prop:lca-impermeab};
(iii)~failure safety with $100\%$ SDR recall and correctly attributed, non-generic requests, complemented by construction-time enforcement of the input type invariant; and 
(iv)~per-protocol independence under multi-protocol activation. These are properties of the orchestration layer: they hold irrespective of the diagnostic model and therefore generalise beyond the particular components used here.

\subsection{Positioning}
\label{subsec:disc-pos}
The impermeability result is the practical payoff of the design. Because the orchestration projection is invariant under model swaps, clinical models can be retrained or replaced without revalidating routing, failure handling, or the output contract---an operation that is structurally impossible in monolithic fusion architectures. The SIP further isolates the framework from downstream IT volatility, so that changes to EMR schemas or interoperability standards do not propagate into any upstream module. Combined with cancer-agnostic switching, this positions the LCA as an orchestration layer rather than a competitor to any individual diagnostic model.

\subsection{Limitations}
\label{subsec:disc-limit}
The evaluation is deliberately scoped. It exercises the deterministic variant V1 only; the probabilistic router V2, although fully specified, is not empirically validated here, and its routing guarantee is conditional on calibration, in line with recent empirical analyses of calibration and selective prediction in multimodal clinical condition classification \cite{lopez2026empirical}. Inference is provided by declared reference stubs emitting simulated diagnostic content — not executed neural architectures — and the remedy module by a declared rule-based stub: this is appropriate for testing orchestration properties but entails no claim about clinical accuracy or about the behaviour of any specific model class, and the latency argument is made by reference to typical inference magnitudes rather than against a production model. Remedy synthesis policies and SDR request templates are protocol-declared and require maintenance as clinical guidelines evolve, and natural-language generation is presently English-centric. Finally, empirical generality is shown on a single real pathology augmented by a mock protocol; protocol generality at scale is argued formally (Q2) rather than measured across many pathologies.

\subsection{Ethical \& regulatory}
\label{subsec:disc-ethic_reg}
The LCA is a decision-support layer, not an autonomous diagnostic agent. Its failure-safety mechanism enforces a human-in-the-loop pathway: missing or invalid evidence halts the affected protocol and emits a targeted SDR rather than a silent or speculative output, the behaviour expected of a safe clinical decision support system. A deployed instantiation operating within a diagnostic pathway would fall under software-as-a-medical-device regimes (e.g., FDA SaMD, EU MDR); the architecture supports this through the SIP audit trail and the exclusion of model identities from the external payload, which together aid traceability and governance. Clinical responsibility for any resulting action remains with the supervising clinician. This human‑in‑the‑loop design is consistent with established CDSS design principles that emphasize mitigation of alert fatigue, algorithmic bias, and preservation of clinician judgment \cite{bayor2025designing}.

\section{Conclusion}
\label{sec:conl}
The integration of artificial intelligence into clinical oncology has historically been constrained by rigid, monolithic models that permanently couple data ingestion with specific neural network inferences. To overcome these systemic bottlenecks, this paper introduced the Large Cancer Assistant (LCA), a model-agnostic orchestration framework designed to asynchronously harmonize multimodal spatial, semantic, and biological patient data. By enforcing the formal principle of Algorithmic Impermeability, the LCA establishes a unidirectional pipeline that structurally decouples raw data routing from black-box deep learning inference, and mathematical predictions from clinical heuristics. Furthermore, the dual-variant routing paradigm---offering both explicit deterministic parameterization (V1) and the formal specification for autonomous probabilistic detection (V2)---ensures the framework remains highly adaptable to varying levels of hospital IT infrastructure fidelity.  

A defining feature of the LCA architecture is its deliberate boundary constraint: the orchestration systematically outputs a proprietary Standardized Intermediate Payload (SIP). This structural design successfully insulates the core AI systems from the volatile protocols of external hospital databases. Consequently, the framework demonstrates immediate clinical applicability for any multimodal oncological Clinical Decision Support System (CDSS) requiring a scalable, model-agnostic orchestration layer that natively provides a robust, interoperability-ready boundary.  

To position the LCA framework within a broader clinical translation trajectory, we define an explicit, three-tiered research roadmap. First, the downstream translation of the SIP into strict HL7 FHIR resources is deliberately scoped as an independent architectural challenge, constituting the exclusive subject of a subsequent dedicated paper. Second, parallel informatics efforts will address the factored extraction of deep learning features across distinct modalities. Finally, future work will focus on the rigorous empirical validation and calibration of the probabilistic routing variant (V2) across large-scale, multi-pathology clinical cohorts.

\appendix
\renewcommand{\thesection}{Appendix \Alph{section}}

\section{Entry Theory Complete Formalization}
\label{app:entry}

\begin{definition}[Domain]
\label{def:entry-domain}
A domain is a pair $(\Omega, \mathfrak{G})$ where $\Omega$ is a non-empty set of positions and $\mathfrak{G}$ is a group acting on $\Omega$ via a left action $\mathfrak{G} \times \Omega \to \Omega$.
\end{definition}

\begin{definition}[Entry]
\label{def:entry-entry}
An entry is a signal $x \in \mathcal{C}^{\Omega}$, i.e., a function $x : \Omega \to \mathcal{C}$, defined on a domain $(\Omega, \mathfrak{G})$ with feature space $\mathcal{C}$. The group $\mathfrak{G}$ acts on entries by $(g \cdot x)(\omega) = x(g^{-1}\omega)$ for all $g \in \mathfrak{G}$, $\omega \in \Omega$. (See details in Table~\ref{tab:app:entry})
\end{definition}

\begin{table*}[width=.9\textwidth, cols=4, pos=tbp]
    \caption{Modality Structural Characterization (T)}
    \label{tab:app:entry}
    \begin{tabular*}{\tblwidth}{@{} LLLL @{}}
        \toprule
        Modality & $\Omega$ & $\mathfrak{G}$ & $\mathcal{C}$ \\
        \midrule
        CT / MRI volume & $\{1,\ldots,D\} \times \{1,\ldots,H\} \times \{1,\ldots,W\}$ & Translations on $\mathbb{Z}^3$ & $\mathbb{R}$ \\
        Histopathology (WSI) & $\{1,\ldots,H\} \times \{1,\ldots,W\}$ & Translations on $\mathbb{Z}^2$ & $\mathbb{R}^3$ \\
        Clinical text & $\{1,\ldots,L\}$ & Translations on $\mathbb{Z}$ & $\Sigma$ \\
        Tabular biological metrics & $\{1,\ldots,d\}$ & $\{e\}$ (trivial) & $\mathbb{R}$ \\
        \bottomrule
    \end{tabular*}
\end{table*}

\subsection*{Canonical Examples (Entries)}
\label{app:subsec:entry-canonical}

\begin{itemize}
    \item \textbf{CT volume.} A computed tomography scan of $D \times H \times W$ voxels: $x : \{1,\ldots,D\} \times \{1,\ldots,H\} \times \{1,\ldots,W\} \to \mathbb{R}$ where $x(d,h,w) \in \mathbb{R}$ is the Hounsfield Unit value, with $\mathfrak{G} = $ translations on $\mathbb{Z}^3$.
    \item \textbf{Histopathology slide (WSI).} A whole-slide image: $x : \{1,\ldots,H\} \times \{1,\ldots,W\} \to \mathbb{R}^3$ where $x(h,w) \in \mathbb{R}^3$ is the RGB vector, with $\mathfrak{G} = $ translations on $\mathbb{Z}^2$.
    \item \textbf{Clinical note.} Tokenized text over vocabulary $\Sigma$: $x : \{1,\ldots,L\} \to \Sigma$, with $\mathfrak{G} = $ translations on $\mathbb{Z}$.
    \item \textbf{Blood panel.} A tabular record of $d$ analytes: $x : \{1,\ldots,d\} \to \mathbb{R}$, with $\mathfrak{G} = \{e\}$ (trivial symmetry group). Scalar values (e.g., PSA level) are represented identically prior to ingestion.
\end{itemize}

\begin{definition}[Clinical History]
\label{def:entry-clinic_hist}
A clinical history associated with patient identifier $s$ is an ordered pair $\mathcal{H} = \bigl((\tilde{x}_j)_{j=1}^{T},\; s\bigr)$ where $(\tilde{x}_j)_{j=1}^{T}$ is a finite ordered sequence of characterized entries, indexed by acquisition order.
\end{definition}

\begin{definition}[LCA Input]
\label{def:entry-lca_input}
An LCA input is a clinical history: $\mathcal{E} := \bigl\{((\tilde{x}_j)_{j=1}^T,\, s) : T \geq 1\bigr\}$.
\end{definition}

\begin{definition}[Provenance]
\label{def:entry-prov}
The provenance of an entry $x$ is a class $\rho \in \mathsf{Prov}$. In the oncology domain:
\begin{align*}
\mathsf{Prov} &= \{\\
& \text{imaging/CT}, \text{imaging/MRI}, \text{imaging/PET},\\
& \text{pathology/WSI}, \text{biology/panel}, \\
& \text{documentation/note},\\
& \text{inference/AI}\\
\}
\end{align*}
\end{definition}

\begin{definition}[Usage]
\label{def:entry-usage}
The usage of an entry $x$ is a non-empty subset $u \subseteq \mathsf{Usg}$. In the oncology domain:
\begin{equation*}
\mathsf{Usg} = \{\text{observation}, \text{diagnosis}, \text{procedure}, \text{medication}\}
\end{equation*}
\end{definition}

\begin{definition}[Epistemic Certainty]
\label{def:entry-epist_cert}
The epistemic certainty of an entry $x$ is a class $\kappa \in \mathsf{Cert}$, where
\begin{equation*}
\mathsf{Cert} = \{\text{confirmed}, \text{suspected}, \text{inferred}\}
\end{equation*}
See details in Table~\ref{tab:app:subsec:entry-canonical}.
\end{definition}

\begin{table}[tbp]
    \caption{Default Medical Signature Mappings : $\kappa$ values}
    \label{tab:app:subsec:entry-canonical}
    \begin{tabular*}{\tblwidth}{@{} L @{}}
        \toprule
        Pathology \& direct measurement prov. default to \textbf{confirmed}. \\
        Documentation-based prov. default to \textbf{suspected}. \\
        Computational inference prov. default to \textbf{inferred}. \\
        \midrule
        \textit{\small \textbf{Notes:}}\\
        \textit{\small $\bullet$ $\rho$ and $\kappa$ are orthogonal.}\\
        \textit{\small $\bullet$ A histopathology report typically carries $\kappa = \text{confirmed}$,}\\
        \textit{\small \hspace{1em} but may carry $\kappa = \text{suspected}$ if tissue is insufficient.} \\
        \bottomrule
    \end{tabular*}
\end{table}

\begin{definition}[Medical Signature]
\label{def:entry-med_sig}
The medical signature of an entry $x$ is the triple $\sigma(x) = (\rho,\; u,\; \kappa) \in \mathsf{Prov} \times (2^{\mathsf{Usg}} \setminus \{\varnothing\}) \times \mathsf{Cert}$.
\end{definition}

\begin{definition}[Characterized Entry]
\label{def:entry-char_entry}
A characterized entry is a tuple $\tilde{x} = \bigl(x,\; \Omega,\; \mathfrak{G},\; \mathcal{C},\; \sigma\bigr)$ where $(x, \Omega, \mathfrak{G}, \mathcal{C})$ constitutes the CS axis and $\sigma$ constitutes the medical axis.
\end{definition}

\begin{proposition}[Independence of Axes]
\label{prop:entry-ind_axe}
The CS axis \\ $(\Omega, \mathfrak{G}, \mathcal{C})$ and the medical axis $\sigma$ are mutually independent.
\end{proposition}

\begin{pf}
Independence in both directions is witnessed by construction:
\begin{enumerate}[1.]
    \item \textbf{CS does not determine medical.} Two entries sharing identical $(\Omega, \mathfrak{G}, \mathcal{C})$ (e.g., both 3D volumes on $\mathbb{Z}^3$ translations with $\mathcal{C} = \mathbb{R}$) may have distinct medical signatures: a conclusive CT carries $\kappa = \text{confirmed}$, while an inconclusive CT of identical structural form carries $\kappa = \text{suspected}$.
    \item \textbf{Medical does not determine CS.} Two entries sharing identical medical signature 
    \begin{equation*}
    \sigma = (\text{biology/panel}, \{\text{observation}\}, \text{confirmed})
    \end{equation*}
    may have distinct CS structures: a standard blood panel is a tabular entry on domain size $d$, while a specialized urine analysis panel has a domain of different size $|\Omega| \neq d$ in general. \hfill $\blacksquare$
\end{enumerate}
\end{pf}

\section{The LCA Framework \& Impermeability}
\label{app:lca}

\begin{definition}[Large Cancer Assistant Framework]
\label{def:lca-lca}
The Large Cancer Assistant framework is a 7-tuple:
$$\text{LCA} \;=\; \bigl(\mathcal{E},\;\mathcal{P},\;f_{DP},\;f_{CS},\;f_{LCD},\; f_{LCR},\;f_{LCW}\bigr)$$
where:
\begin{itemize}
    \item $\mathcal{E}$: the input space (Definition 4).
    \item $\mathcal{P} = \{p_1, \ldots, p_K\}$: the protocol catalog, shared across all modules.
    \item $f_{DP}$: the Data Preprocessing Module (DPM).
    \item $f_{CS}$: the Cancer Switching Module (CSM).
    \item $f_{LCD}$: the Large Cancer Diagnostic Module (LCDM).
    \item $f_{LCR}$: the Large Cancer Remedy Module (LCRM).
    \item $f_{LCW}$: the Large Cancer Wordings Module (LCWM).
\end{itemize}
\end{definition}

\begin{proposition}[System Impermeability]
\label{prop:lca-impermeab}
Let $\Theta^{(k)}$ be the space of parameter configurations for protocol $p_k$. For any two parameter configurations $\theta, \theta' \in \Theta^{(k)}$ such that $f^{(k)}_{LCD,\theta}$ and $f^{(k)}_{LCD,\theta'}$ both satisfy the interface contract $\mathcal{I}^{(k)}_{LCD}$ (Definition 25), and similarly $f^{(k)}_{LCR,\theta}$ and $f^{(k)}_{LCR,\theta'}$ satisfy $\mathcal{I}^{(k)}_{LCR}$:

let $\pi$ denote the projection of the SIP onto its orchestrat\-ion-structural component --- the activation set $\hat{P}$, the per-protocol routing decisions, the set of protocols emitting $\bot$, and the SIP schema --- discarding the inference content $\tilde{y}^{(k)}$ and the generated text. Then:
$$\pi\bigl(f_{LCA}[\theta](e)\bigr) = \pi\bigl(f_{LCA}[\theta'](e)\bigr) \qquad \forall\, e \in \mathcal{E}$$
i.e., module activation, routing, failure handling, and SIP structure are identical.
\end{proposition}

The inference content $\tilde{y}^{(k)}$ and the generated text may differ between implementations; the orchestration-structural projection does not. This follows : 
\begin{equation}
\pi(f_{LCA}\theta) = \pi(f_{LCA}\theta') \quad \forall e \in \mathcal{E}
\end{equation}

\begin{pf}
\label{proof:lca-impermeab}
Let $e \in \mathcal{E}$ be arbitrary, and fix any pair $(\theta, \theta')$ of parameter configurations satisfying the stated interface contracts. We trace the pipeline under both configurations and show that every component retained by $\pi$ is identical.

\smallskip
\noindent\textit{Step~1 --- DPM and CSM are $\theta$-free.}
$f_{DP}$ and $f_{CS}$ carry no dependence on $\theta$ or $\theta'$. Therefore $\mathcal{D}_{out} = f_{DP}(e)$ and $\hat{P} = f_{CS}(\mathcal{D}_{out},\mu)$ are identical under both configurations. The activation set, the per-protocol routing decisions, and the induced partition $(\hat{P}_{LCD}, \hat{P}_{LCR})$ are the same.

\smallskip
\noindent\textit{Step~2 --- $\bot$-status of LCDM instances.}
Fix any $p_k \in \hat{P}_{LCD}$. By assumption, $f^{(k)}_{LCD,\theta}$ and $f^{(k)}_{LCD,\theta'}$ both satisfy $\mathcal{I}^{(k)}_{LCD}$ (Definition~\ref{def:lcdm_lcrm-mod_inter}). The failure precondition set $\mathcal{D}^{(k)}_{\bot}$ is a declared component of the contract $\mathcal{I}^{(k)}_{LCD}$, shared by every conforming implementation. Definition~\ref{def:lcdm_lcrm-mod_inter} requires that a conforming module emits $\bot$ \emph{if and only if} its input belongs to $\mathcal{D}^{(k)}_{\bot}$. 
Since both implementations receive the identical input $(\mathcal{D}_{out},\hat{P})$ established in Step~1:
\begin{itemize}
  \item $(\mathcal{D}_{out},\hat{P})\in\mathcal{D}^{(k)}_{\bot}$:
        both emit $\bot$;
  \item $(\mathcal{D}_{out},\hat{P})\notin\mathcal{D}^{(k)}_{\bot}$:
        both emit outputs in $\mathcal{Y}^{(k)}_{LCD}$
        (inference content may differ).
\end{itemize}
The $\bot$-status of every LCDM instance is therefore identical under $\theta$ and $\theta'$.

\smallskip
\noindent\textit{Step~3 --- $\bot$-status of LCRM instances.}
Identical argument for $f^{(k)}_{LCR,\theta}$ and $f^{(k)}_{LCR,\theta'}$ with respect to $\mathcal{I}^{(k)}_{LCR}$.

\smallskip
\noindent\textit{Step~4 --- SIP schema.}
The SIP schema is determined by $\hat{P}$, $\mathrm{support}(p_k)$, and $\mathrm{type}(p_k)$ --- all declared a priori as protocol parameters, independent of $\theta$. It is therefore identical under both configurations.

\smallskip
\noindent\textit{Conclusion.}
The projection $\pi$ retains: (i)~the activation set $\hat{P}$, (ii)~per-protocol routing decisions, (iii)~the set of protocols emitting $\bot$, and (iv)~the SIP schema; it discards inference content $\tilde{y}^{(k)}$ and generated text.
By Steps~1--4, each retained component is identical under $\theta$ and $\theta'$. Hence $\pi(f_{LCA}[\theta](e))=\pi(f_{LCA}[\theta'](e))$.
\end{pf}

\begin{corollary}
\label{cor:lca-unchange}
The modules $f_{DP}$, $f_{CS}$, and $f_{LCW}$ are unchanged by any update to $f_{LCD}^{(k)}$ or $f_{LCR}^{(k)}$. Clinical AI models can be retrained, updated, or replaced independently of the orchestration layer.
\end{corollary}

\subsection*{Definition of the LCA on $\mathcal{E}$.} 
\label{app:subsec:lca-q1_detailed}
$f_{LCA}$ is defined for every $e \in \mathcal{E}$, including partial-modality inputs. No input causes a silent failure. Under V2, a degenerate history whose entries carry no clinical signal (all per-protocol scores below $\tau$) causes the CSM to emit $\bot$, returning a top-level SDR without invoking any inference module. Under V1, the CSM returns $P_\lambda$ regardless of signal quality; any $\bot$ originates from the LCDM or LCRM at the protocol level, not from the CSM.

\section{Data Preprocessing Module (DPM) Formalization}
\label{app:dpm}

\begin{definition}[Modality-Specific Preprocessor]
\label{def:dpm-modal_prep}
For each $\rho\in\mathsf{Prov}$, a modality-specific preprocessor is a function
\begin{equation*}
    \phi_\rho \;:\; (\mathcal{C}_\rho)^{\Omega_\rho}
  \;\longrightarrow\; (\mathcal{C}_\rho')^{\Omega_\rho'}
\end{equation*}
mapping entries on source domain $(\Omega_\rho,\mathfrak{G}_\rho, \mathcal{C}_\rho)$ to entries on target domain $(\Omega_\rho', \mathfrak{G}_\rho,\mathcal{C}_\rho')$. Each $\phi_\rho$ is defined independently for its provenance class; no preprocessor is shared across distinct classes $\rho\neq\rho'$.
\end{definition}

\begin{rmk}
\label{rmk:dpm-domain}
A preprocessor $\phi_\rho$ may act on the domain $\Omega_\rho$ (e.g., isotropic resampling of CT volumes from anisotropic voxel spacing $s=(s_x,s_y,s_z)$ to canonical spacing $s'=(1,1,1)$\,mm), on the feature space $\mathcal{C}_\rho$ (e.g., Hounsfield Unit clipping to $[-1000,400]$ followed by min-max normalization to $[0,1]$), or on both simultaneously. Domain transformation proceeds via interpolation in the continuous embedding of $\Omega_\rho$; feature transformation is a pointwise or neighborhood-dependent map on $\mathcal{C}_\rho$.
\end{rmk}

\begin{definition}[Data Preprocessing Module ($f_{DP}$)]
\label{def:dpm-dpm}
The \\ DPM function $f_{DP}$ is defined over any LCA input (Definition~\ref{def:entry-lca_input}): a clinical history of characterized entries $\tilde{\mathcal{H}}= ((\tilde{x}_j)_{j=1}^T,\,s)$ with $T\geq 1$, together with associated metadata $\mu\in\mathcal{M}$. It is the sequential composition of five functions $f_0,f_1,f_2,f_3,f_4$.

\textbf{$f_0$ --- History sequencing.} $f_0$ decomposes the input history into an index-tagged family and extracts the ordering metadata:
\begin{equation*}
  f_0(\tilde{\mathcal{H}}) \;=\;\bigl((j,\tilde{x}_j)_{j=1}^T,\;\; \mathcal{O}=((j)_{j=1}^T,\,s)\bigr)
\end{equation*}
The ordering metadata $\mathcal{O}$ records the acquisition indices $(j)_{j=1}^T$ and the subject identifier $s$. It is carried separately through $f_1$, $f_2$, $f_3$ without modification and consumed by $f_4$.
The index tag $j$ preserves the element-to-index association through all stages, so that distinct entries sharing an identical value or signature remain distinguishable.
  
\textbf{$f_1$ --- Signature extraction and preprocessing routing.}
\begin{equation*}
  f_1\!\left((j,\tilde{x}_j)_{j=1}^{T}\right)   \;=\;\Bigl(\, j,x_j,\Omega_j,\mathfrak{G}_j,\mathcal{C}_j,\rho_j)        _{j=1}^{T},\;\;(j,u_j,\kappa_j)_{j=1}^{T}\,\Bigr)
\end{equation*}
$f_1$ decomposes each characterized entry into two streams. The first stream $\mathcal{S}_{CS}=(j,x_j,\Omega_j,\mathfrak{G}_j,\mathcal{C}_j, \rho_j)_{j=1}^{T}$ carries the structural CS components together with the provenance $\rho_j$, extracted from $\sigma_j$, which serves as the preprocessing routing key. The second stream $\mathcal{S}_{med}= (j,u_j,\kappa_j)_{j=1}^{T}$ carries the clinical attributes of the medical signature. The full signature $\sigma_j=(\rho_j,u_j,\kappa_j)$ is never modified; it is reassembled by $f_3$.

\textbf{$f_2$ --- Structural preprocessing.}
\begin{equation*}
    f_2(\mathcal{S}_{CS}) \;=\;\bigl(j,\,\phi_{\rho_j}(x_j),\;\Omega_{\rho_j}',\;         \mathfrak{G}_{\rho_j},\;\mathcal{C}_{\rho_j}',\;\rho_j \bigr)_{j=1}^{T}
\end{equation*}
$f_2$ applies the modality-specific preprocessor $\phi_{\rho_j}$, selected via the routing key $\rho_j\in\mathcal{S}_{CS}$, to each CS component independently. The functions $\{\phi_{\rho_j}\}$ are applied without defined inter-entry ordering; within each $\phi_{\rho_j}$, internal sub-operations are applied sequentially as specified by the preprocessor for that provenance class.

\textbf{$f_3$ --- Axis reconstruction.}
\begin{align*}
f_3\!\left(\bigl(j,x_j',\Omega_j',\mathfrak{G}_j,\mathcal{C}_j', \rho_j\bigr)_{j=1}^{T},\;(j,u_j,\kappa_j)_{j=1}^{T}\right) \;\\
=\;\bigl(j,x_j',\Omega_j',\mathfrak{G}_j,\mathcal{C}_j',    \sigma_j\bigr)_{j=1}^{T}
\end{align*}
where $\sigma_j=(\rho_j,u_j,\kappa_j)$ is reconstructed from the routing key $\rho_j$ carried in $\mathcal{S}_{CS}$ and the clinical attributes $(u_j,\kappa_j)$ carried in $\mathcal{S}_{med}$. $f_3$ takes as input both the output of $f_2$ and $\mathcal{S}_{med}$ retained from $f_1$.

\textbf{$f_4$ --- History recomposition.}
$f_4$ restores the acquisition ordering and subject identity:
\begin{equation*}
    f_4\!\left((j,\tilde{x}_j')_{j=1}^T,\;\mathcal{O}\right) \;=\;\bigl((\tilde{x}_j')_{j=1}^T,\;\;s\bigr)
\end{equation*}
$f_4$ reorders the preprocessed entries by their acquisition index $j$ and reattaches the subject identifier $s$, yielding a preprocessed clinical history. The temporal ordering established by $f_0$ is exactly preserved.

\textbf{Full composition.}
Let $(C,\mathcal{O})=f_0(\tilde{\mathcal{H}})$ \\
and $(\mathcal{S}_{CS}, \mathcal{S}_{med})=f_1(C)$. Then:
\begin{equation*}
    f_{DP}(e,\mu)\;=\;\Bigl(\,f_4\bigl(f_3(f_2(\mathcal{S}_{CS}),\,\mathcal{S}_{med}), \,\mathcal{O}\bigr),\;\;\mu\,\Bigr)
\end{equation*}
The ordering metadata $\mathcal{O}$ produced by $f_0$ is passed directly to $f_4$ without being modified by $f_1$, $f_2$, or $f_3$.
\end{definition}

\subsection*{Properties of $f_{DP}$}
\label{app:desc:dpm-pptes}
\begin{description}
  \item[(P1.1) Symmetry group preservation.] The symmetry gro\-up is
        preserved: $\mathfrak{G}_{\rho_i}'=\mathfrak{G}_{\rho_i}$. A
        signal defined on a domain with group $\mathfrak{G}_{\rho_i}$
        before preprocessing remains defined on a domain with the same
        group after preprocessing.
  \item[(P1.2) Medical transparency.] The medical signature
        $\sigma_i=(\rho_i,u_i,\kappa_i)$ is never modified during
        preprocessing. The clinical attributes $u_i$ and $\kappa_i$ are
        not accessed by $f_2$. Preprocessing is parameterized exclusively
        by $\rho_i$, which serves as a routing key selecting the
        appropriate canonical form. Two arguments motivate this design:
        (i)~\textit{Semantic invariance of the medical signature.} The
        signature $\sigma=(\rho,u,\kappa)$ characterizes provenance,
        usage, and epistemic certainty---properties intrinsic to the data
        item and invariant under changes of representation. A resampled
        and normalized CT volume retains the same $\rho$, $u$, and
        $\kappa$. Preprocessing alters the representation of the signal,
        not its clinical meaning. (ii)~\textit{Parameterization by type,
        not by clinical judgment.} Preprocessing depends on the medical
        axis exclusively through $\rho_i$. It does not depend on $u_i$
        or $\kappa_i$, which carry clinical judgment. Proposition~1 is
        unaffected.
  \item[(P1.3) Metadata invariance.] The metadata $\mu$ passes thr\-ough
        $f_{DP}$ unmodified.
  \item[(P1.4) Identity as degenerate case.] When $\Omega_i=\Omega_{\rho_i}'$
        and $\mathcal{C}_i=\mathcal{C}_{\rho_i}'$, the preprocessor
        $\phi_{\rho_i}$ reduces to the identity map:
        $\phi_{\rho_i}(x_i)=x_i$.
\end{description}

\subsection*{Canonical Examples (Preprocessing).}
\label{app:desc:dpm-can_exples}
\textbf{\textit{Natural image.}}\\
$\bullet$ Source: $x\in\{0,\ldots,255\}^{\{1,\ldots,H\}\times \{1,\ldots,W\}}$.\\
$\bullet$ Target: $x'\in[0,1]^{\{1,\ldots,224\}\times \{1,\ldots,224\}}$.\\
$\bullet$ $\phi_{\text{img}}$ applies: (1)~bilinear resampling to $224\times 224$; (2)~normalization $x'(h,w)=x_{\mathrm{resized}}(h,w)/255$. Both $\Omega$ and $\mathcal{C}$ change.

\textit{\textbf{CT medical volume.}}\\
$\bullet$ Source: $x\in\mathbb{R}^{\{1,\ldots,D\}\times\{1,\ldots,H\} \times\{1,\ldots,W\}}$ with spacing $s=(s_x,s_y,s_z)$\,mm.\\
$\bullet$ Target: $x'\in[0,1]^{D'\times H'\times W'}$ with isotropic spacing $s'=(1,1,1)$\,mm, where $D'=\mathrm{round}(D\cdot s_z)$, $H'=\mathrm{round}(H\cdot s_y)$, $W'=\mathrm{round}(W\cdot s_x)$.\\
$\bullet$ $\phi_{\mathrm{CT}}$ applies sequentially: (1)~trilinear resampling; (2)~HU clipping $v\mapsto\max(\min(v,400),-1000)$; (3)~normalization $v\mapsto(v+1000)/1400$. Both $\Omega$ and $\mathcal{C}$ change.
        
\textit{\textbf{Clinical note.}}\\
$\bullet$ Source: $x\in\Sigma^{\{1,\ldots,L\}}$. By default $\phi_{\text{text}}$ is the \textbf{identity}.\\
$\bullet$ Classical pipelines may apply a TF-IDF map $\Sigma^{\{1,\ldots,L\}}\to \mathbb{R}^{\{1,\ldots,d\}}$, in which case both $\Omega$ and $\mathcal{C}$ change.

\textit{\textbf{Blood panel.}}\\
$\bullet$ Source: $x\in\mathbb{R}^{\{1,\ldots,d\}}$ in physical units.\\
$\bullet$ Target: $x'\in[0,1]^{\{1,\ldots,d\}}$. \\
$\bullet$ $\phi_{\text{bio}}$ normalizes each analyte $k$ against its reference range: $x'(k)=\mathrm{clip}\!\left(\tfrac{x(k)-r_k^{\min}} {r_k^{\max}-r_k^{\min}},0,1\right)$. Only $\mathcal{C}$ changes.

\textit{\textbf{Identity case.}} When $\Omega_i=\Omega_{\rho_i}'$ and $\mathcal{C}_i= \mathcal{C}_{\rho_i}'$: $\phi_{\rho_i}(x_i)=x_i$. Neither $\Omega$ nor $\mathcal{C}$ changes.

\section{Cancer Switching Module (CSM) Complete Formalization}
\label{app:csm}

\begin{definition}[Protocol Catalog]
\label{def:csm-csm}
A protocol catalog is a finite, non-empty set $\mathcal{P} = \{p_1, \ldots, p_K\}$ where each $p_k$ designates a cancer-specific diagnostic and therapeutic pipeline. $\mathcal{P}$ is a parameter of the LCA framework, declared a priori from clinical scope; it is not inferred from data.
\end{definition}

\begin{definition}[Activation Set]
\label{def:csm-act_set}
An activation set is either a non-empty subset $\hat{P} \subseteq \mathcal{P}$ or $\bot$, where $\hat{P}$ designates the set of cancer-specific modules activated for investigation, and $\bot$ designates an activation failure requiring human intervention (Supplementary Data Request).
\end{definition}

\begin{definition}[CSM Input]
\label{def:csm-csm_input}
The CSM receives the output of $f_{DP}$ (Definition \ref{def:dpm-dpm}): a preprocessed clinical history $\tilde{\mathcal{H}}' = ((\tilde{x}_j')_{j=1}^T,\, s)$ together with associated metadata $\mu$. Each preprocessed entry $\tilde{x}_j'$ is a characterized entry in the sense of Definition \ref{def:entry-char_entry}, carrying both CS structure and medical signature. Denote by $e \in \mathcal{X}_{V2}$ the full DPM output. The CSM function $f_{CS}$ maps $(e, \mu)$ to an activation set in $2^\mathcal{P} \setminus \{\varnothing\}$ or $\bot$.
\end{definition}

\subsection*{Variant V1 --- Deterministic Activation}
\label{app:subsec:csm-V1}

\begin{definition}[A Priori Activation Parameter]
\label{def:csm-csm-V1-p_priori}
The V1 \-activation parameter is a direct protocol set declaration: $P_\lambda \subseteq \mathcal{P}, \quad P_\lambda \neq \varnothing$. $P_\lambda$ specifies the set of cancer-specific modules to activate unconditionally. It is declared a priori --- by the clinician or the upstream system --- prior to any signal processing; it is not computed from data.
\end{definition}

\begin{definition}[V1 Activation Function]
\label{def:csm-csm-V1-act_func}
The V1 activation function is: $R_{V1} = P_\lambda$ if $P_\lambda$ is declared and $P_\lambda \subseteq \mathcal{P}$, $P_\lambda \neq \varnothing$; $R_{V1} = \bot$ otherwise. V1 does not access $e$, $\mathcal{S}_{CS}$, $\mathcal{S}_{med}$, or $\mu$. The activation set is determined entirely by the a priori parameter $P_\lambda$.
\end{definition}

\begin{rmk}[Activation parameter vs. safety threshold]
\label{rmk:csm-csm-V1-disctinction}
The activation parameter $P_\lambda \subseteq \mathcal{P}$ (V1) is a content parameter: it determines which modules are activated. The confidence threshold $\tau \in (0, 1)$ (V2) is an operational parameter: it determines the minimum per-protocol confidence required before any module is activated.
\end{rmk}

\subsection*{Variant V2 --- Probabilistic Activation}
\label{app:subsec:csm-V2}

\begin{definition}[Probability Simplex]
\label{def:csm-csm-V2-prob_simp}
For a finite set $S$, the probability simplex over $S$ is:
\begin{equation*}
    \Delta(S) = \left\{q \in \mathbb{R}^{|S|} \;:\; q_s \geq 0 \;\forall s,\quad \sum_{s \in S} q_s = 1\right\}
\end{equation*}
\end{definition}

\begin{definition}[V2 Input Space and Hyp. Class]
\label{def:csm-csm-V2-hyp_cases}
V2 input space is $\mathcal{X}_{V2}$, the set of all valid DPM outputs. The hypothesis class $\mathcal{H}$ is the set of measurable functions $h : \mathcal{X}_{V2} \to [0,1]^K$, where $h(e)_k \in [0,1]$ is the per-protocol confidence score that $p_k$ is relevant to case $e$. The framework imposes no invariance or equivariance constraint on $h$: by algorithmic impermeability, the internal architecture is not part of the V2 contract.
\end{definition}

\begin{definition}[Routing Risk]
\label{def:csm-csm-V2-rout_risk}
Let $\mathcal{D}$ be the unknown joint distribution over $\mathcal{X}_{V2} \times 2^\mathcal{P}$. The routing risk of a hypothesis $h \in \mathcal{H}$ is:
\begin{equation*}
R(h) \;=\; \mathbb{E}_{(e,\, P^*) \sim \mathcal{D}}
\left[\,\sum_{k=1}^{K}
\mathbf{1}\bigl[(p_k \in \hat{P}(h,e)) \;\neq\; (p_k \in P^*)\bigr]
\,\right]
\end{equation*}
where $P^* \subseteq \mathcal{P}$ is the ground-truth set of relevant protocols, and $\hat{P}(h, e) = \{p_k : h(e)_k \geq \tau\}$.
\end{definition}

\begin{definition}[Calibration]
\label{def:csm-csm-V2-calib}
A hypothesis $h \in \mathcal{H}$ is calibrated with respect to $\mathcal{D}$ if, for all $k \in \{1,\ldots,K\}$ and all $c \in [0,1]$:
\begin{equation*}
P_{(e,\, P^*) \sim \mathcal{D}}\!\left(p_k \in P^* \;\middle|\; h(e)_k = c\right) \;\approx\; c    
\end{equation*}
$\mathcal{D}$-almost everywhere.
\end{definition}

\begin{definition}[V2 Activation Function]
\label{def:csm-csm-V2-act_func}
The V2 activation function is defined by a hypothesis $h \in \mathcal{H}$ approximately minimizing $R$ subject to calibration, with per-protocol threshold gating at $\tau \in (0,1)$:
$$\hat{P}(e) \;=\; \bigl\{p_k \in \mathcal{P} \;:\; h(e)_k \geq \tau\bigr\}$$
$$R_{V2}(e)
= \begin{cases}
\hat{P}(e) & \text{if } \hat{P}(e) \neq \varnothing \\[4pt]
\bot & \text{if } \hat{P}(e) = \varnothing
\end{cases} .$$
\end{definition}

\begin{definition}[CSM Function]
\label{def:csm-csm-V2-csm_func}
The CSM function $f_{CS}$ selects an activation variant and emits an activation set:
$$f_{CS}(e,\; \mu)
= \begin{cases}
R_{V1} = P_\lambda & \text{(V1: parameterized)} \\[4pt]
R_{V2}(e) & \text{(V2: auto-detecting)}
\end{cases} .$$
\end{definition}

\subsection*{Properties of $f_{CS}$}
\label{app:subsec:csm-pptes}

\begin{enumerate}[\bf (P2.1)]
    \item \textbf{Output type.} For both variants, $f_{CS}$ emits an activation set in $2^\mathcal{P} \setminus \{\varnothing\}$ or $\bot$.
    \item \textbf{Symmetry as optional structural prior.} An implementation of $h$ may be designed $\mathfrak{G}_{t_i}$-invariant, but this is an architectural choice, not a framework guarantee.
    \item \textbf{Activation parameter primacy in V1.} V1 does not access $e$, $\mathcal{S}_{CS}$, $\mathcal{S}_{med}$, $\mu$, or temporal structure.
    \item \textbf{Dual-axis activation in V2.} V2 uses both the CS components and the medical signatures (including $\kappa$) from $e$.
    \item \textbf{Safety under insufficient evidence.} Both variants emit $\bot$ when activation cannot be committed, terminating without activating downstream modules and triggering a Supplementary Data Request.
\end{enumerate}

\section{LCDM \& LCRM Formalization}
\label{app:lcdm_lcrm}

\begin{definition}[Output Space Types]
\label{def:lcdm_lcrm-out_space}
Two classes of output are admitted for AI modules within the LCA framework:
\begin{enumerate}[1.]
    \item \textbf{Structural output.} A signal $y_{struct} \in (\mathcal{C}_{out}^{(k)})^{\Omega_{out}^{(k)}}$ defined on an output domain ($\Omega_{out}^{(k)}, \mathfrak{G}^{(k)}$) with output feature space $\mathcal{C}_{out}^{(k)}$, both declared a priori as part of the protocol $p_k$ specification.
    \item \textbf{Decisional output.} A probability distribution $y_{dec} \in \Delta(\mathcal{L}^{(k)})$ over a finite, protocol-specific class catalog $\mathcal{L}^{(k)}$, declared a priori.
\end{enumerate}
For each protocol $p_k$, the admissible output type is declared a priori:
\begin{equation*}
\text{type}(p_k) \;\in\; \left\{\;\prod_{i=1}^{m}\mathcal{Y}_{struct}^{(i)},\;\; \mathcal{Y}_{dec},\;\; \prod_{i=1}^{m}\mathcal{Y}_{struct}^{(i)} \times \mathcal{Y}_{dec}\;\right\},
\end{equation*}
where $m \geq 1$
\end{definition}

\begin{definition}[Module Interface Contract]
\label{def:lcdm_lcrm-mod_inter}
An interface contract $\mathcal{I}_\tau^{(k)}$ for a module of type $\tau \in \{\text{LCD}, \text{LCR}\}$ and protocol $p_k \in \mathcal{P}$ specifies:
\begin{itemize}
    \item An input type $\mathcal{X}_\tau^{(k)}$.
    \item An output type $\mathcal{Y}_\tau^{(k)} \in \text{type}(p_k)$, declared a priori.
    \item A failure precondition set $\mathcal{D}_\bot^{(k)} \subseteq \mathcal{X}_\tau^{(k)}$, declared a priori, enumerating the input-level conditions under which $\bot$ is emitted.
\end{itemize}
\end{definition}

\begin{rmk}
\label{rmk:lcdm_lcrm-fail}
When internal inference fails on an input not in $\mathcal{D}_\bot^{(k)}$, a conforming implementation produces a decisional output with a protocol-declared failure code $\mathtt{CODE\_FAIL}^{(k)} \in \mathcal{L}^{(k)}$. $\bot$ is reserved exclusively for inputs in $\mathcal{D}_\bot^{(k)}$.
\end{rmk}

\begin{definition}[Algorithmic Impermeability]
\label{def:lcdm_lcrm-alg_imp}
A module\\
$f^{(k)}$ is algorithmically impermeable with respect to $\mathcal{I}_\tau^{(k)}$ if for any two implementations $f^{(k)}_\theta$ and $f^{(k)}_{\theta'}$ both satisfying $\mathcal{I}_\tau^{(k)}$:
$$\mathrm{Orch}\bigl[f^{(k)}_\theta\bigr](e) \;=\; \mathrm{Orch}\bigl[f^{(k)}_{\theta'}\bigr](e) \qquad \forall\, e \in \mathcal{X}_{V2}$$
where $\mathrm{Orch}[\,\cdot\,]$ denotes the global orchestration function.
\end{definition}

\subsection*{Protocol Catalog Completeness}
\label{app:subsec:lcdm_lcrm-catalog}

\begin{definition}[Protocol Support]
\label{def:lcdm_lcrm-catalog-prot_supp}
For each protocol $p_k \in \mathcal{P}$, its support is a non-empty subset 
$$\text{support}(p_k) \subseteq \{\text{LCD},\; \text{LCR}\}$$ declaring which AI modules are required to handle $p_k$. The three configurations are $\{\text{LCD},\, \text{LCR}\}$, $\{\text{LCD}\}$, and $\{\text{LCR}\}$.
\end{definition}

\begin{proposition}[Catalog Completeness]
\label{def:lcdm_lcrm-catalog-cat_comp}
The protocol catalog $\mathcal{P}$ is complete if and only if
\begin{equation*}
    \forall\, p_k \in \mathcal{P},\quad \mathrm{support}(p_k) \neq \varnothing
\end{equation*}
\end{proposition}

\subsection*{Large Cancer Diagnostic Module (LCDM)}
\label{app:subsec:lcdm_lcrm-lcdm}

\begin{definition}[LCDM Input and Activated Subset]
\label{def:lcdm_lcrm-lcdm-lcdm_act}
Let \\
$\hat{P}$ be the activation set emitted by $f_{CS}$. Define $\hat{P}_{LCD} = \{p_k \in \hat{P} : \text{LCD} \in \text{support}(p_k)\}$. The LCDM interface for protocol $p_k$ is:
$$\mathcal{I}_{LCD}^{(k)} \;:\; \mathcal{X}_{LCD}^{(k)} = \mathcal{X}_{V2} \times 2^\mathcal{P} \;\longrightarrow\; \tilde{\mathcal{Y}}_{LCD}^{(k)} \cup \{\bot\}$$
\end{definition}

\begin{definition}[LCDM Output as Characterized Entry]
\label{def:lcdm_lcrm-lcdm-out}
For each $p_k \in \hat{P}_{LCD}$, $f_{LCD}^{(k)}$ produces
$$\tilde{y}_{LCD}^{(k)} = (y_{LCD}^{(k)},\; \sigma_{LCD}^{(k)})$$
where $\sigma_{LCD}^{(k)} = (\rho_{LCD},\; u_{LCD}^{(k)},\; \kappa_{LCD}^{(k)})$. The epistemic certainty $\kappa_{LCD}^{(k)} = \mathrm{synth}_k((\kappa_i)_i)$ is derived via a declared synthesis policy, defaulting to $\text{inferred}$. The LCDM function is thus :
$$f_{LCD}(\mathcal{D}_{out},\; \hat{P}) \;=\; \Bigl(\,\bigl\{\bigl(\tilde{y}_{LCD}^{(k)},\; p_k\bigr)\bigr\}_{p_k \in \hat{P}_{LCD}},\;\; \mathcal{D}_{out}\,\Bigr)$$
\end{definition}

\subsubsection*{Properties of $f_{LCD}$}
\label{app:subsubsec:lcdm_lcrm-lcdm-pptes}
\begin{enumerate}[\bf (P3.1)]
    \item \textbf{Algorithmic impermeability:} Implementations satisfying $\mathcal{I}_{LCD}^{(k)}$ are interchangeable.
    \item \textbf{Per-protocol independence:} Instances \\
    $\{f_{LCD}^{(k)}\}_{p_k \in \hat{P}_{LCD}}$ are independent at the framework level.
    \item \textbf{SDR trigger:} Inputs in $\mathcal{D}_\bot^{(k)}$ emit $\bot$ mapped to an SDR.
\end{enumerate}

\subsection*{Large Cancer Remedy Module (LCRM)}
\label{app:subsec:lcdm_lcrm-lcrm}

\begin{definition}[LCRM Input and Activated Set]
\label{def:lcdm_lcrm-lcrm-input}
Define 
$$\hat{P}_{LCR} = \{p_k \in \hat{P} : \text{LCR} \in \text{support}(p_k)\}$$
The LCRM interface is:
\begin{align*}
    \mathcal{I}_{LCR}^{(k)} \;:\; \mathcal{X}_{LCR}^{(k)} = \bigl(\tilde{\mathcal{Y}}_{LCD}^{(k)} \times \mathcal{X}_{V2}\bigr) \times 2^\mathcal{P} \\
    \;\longrightarrow\; \mathcal{Y}_{LCR}^{(k)} \cup \{\bot\}
\end{align*}
\end{definition}

\begin{definition}[LCRM Function]
\label{def:lcdm_lcrm-lcrm-func}
The LCRM function is:
\begin{align*}
f_{LCR}\!\Bigl(\bigl(\{(\tilde{y}_{LCD}^{(k)}, p_k)\},\; \mathcal{D}_{out}\bigr),\; \hat{P}\Bigr) \\
\;=\; \Bigl(\,\bigl\{\bigl(\tilde{y}_{LCR}^{(k)},\; p_k\bigr)\bigr\}_{p_k \in \hat{P}_{LCR}},\;\; \mathcal{D}_{out}\,\Bigr)
\end{align*}
\end{definition}

\begin{rmk}[Conditional activation]
\label{rmk:lcdm_lcrm-lcrm-cond_act}
The LCRM is activated only for protocols where $\text{LCR} \in \text{support}(p_k)$. For protocols with $\text{support}(p_k) = \{\text{LCD}\}$, no LCRM instance is created and the LCWM operates solely on LCDM outputs.
\end{rmk}

\subsubsection*{Properties of $f_{LCR}$ \& Canonical Lifting}
\label{app:subsubsec:lcdm_lcrm-lcrm-pptes}
\begin{enumerate}
    \item[$\bullet$] \textbf{Properties (P4.1 - P4.3):} Identical impermeability, independence, and SDR triggers as the LCDM.
    \item[\bf (P4.4)] \textbf{Conditional activation:} For protocols where \\
    $\text{support}(p_k) = \{\text{LCR}\}$, a canonical lifting $\iota$ maps relevant confirmed historical entries (with $u \ni \text{diagnosis}$ and $\kappa = \text{confirmed}$) into the LCDM-output type $\tilde{\mathcal{Y}}_{LCD}^{(k)}$, ensuring the LCRM always receives an input of its declared type. The pipeline remains unidirectional.
\end{enumerate}

\section{LCWM, SDR, and SIP Formalization}
\label{app:lcwm_sdr_sip}

\begin{definition}[LCWM Input]
\label{def:lcwm_sdr_sip-input}
The LCWM receives four inputs:
\begin{itemize}
    \item $\hat{P} \subseteq \mathcal{P}$: the activation set emitted by $f_{CS}$, identifying which protocols were activated and their support configurations.
    \item \textbf{LCDM outputs:} $\{(\tilde{y}_{LCD}^{(k)}, p_k)\}_{p_k \in \hat{P}_{LCD}}$ --- one characterized diagnostic entry per activated diagnostic protocol, or $\bot$ per protocol on LCDM failure.
    \item \textbf{LCRM outputs:} $\{(\tilde{y}_{LCR}^{(k)}, p_k)\}_{p_k \in \hat{P}_{LCR}}$ --- one characterized remedy entry per activated remedy protocol, or $\bot$ per protocol on LCRM failure.
    \item $\mathcal{D}_{out}$: the preprocessed history, bypassed through LCDM and LCRM unchanged.
\end{itemize}
If the CSM itself emitted $\bot$ (activation failure), no downstream modules ran; the LCWM receives only $\hat{P} = \bot$ and generates a top-level SDR.
\end{definition}

\begin{definition}[LCWM Output Components]
\label{def:lcwm_sdr_sip-comp}
The LCWM produces two distinct output components:
\begin{enumerate}[(i)]
    \item \textbf{Text output.} Generated by the LCWM's NLG model from the aggregated inputs. For each protocol $p_k \in \hat{P}$: if upstream outputs are valid, the text contains diagnostic and therapeutic narrative; if any upstream module emitted $\bot$, the text contains a Supplementary Data Request (Definition \ref{def:lcwm_sdr_sip-sdr}) for that protocol. A protocol may produce both narrative and SDR text simultaneously (e.g., valid LCDM output and $\bot$ from LCRM).
    \item \textbf{Structured bypass.}\\
    The LCDM outputs $\{(\tilde{y}_{LCD}^{(k)}, p_k)\}_{p_k \in \hat{P}_{LCD}}$,\\
    LCRM outputs $\{(\tilde{y}_{LCR}^{(k)}, p_k)\}_{p_k \in \hat{P}_{LCR}}$, and the preprocessed history $\mathcal{D}_{out}$ passed through unchanged. This bypass preserves the full structured AI outputs for clinical archiving, audit, and downstream integration, independently of the text generation.
\end{enumerate}
\end{definition}

\begin{definition}[Supplementary Data Request]
\label{def:lcwm_sdr_sip-sdr}
Two SDR variants are defined, depending on which module emitted $\bot$:
\begin{enumerate}[(i)]
    \item \textbf{Top-level SDR} --- emitted when the CSM emits $\bot$ (no protocol was activated):
    $$\text{SDR}_{top} \;=\; \bigl(\text{source} = \text{CSM},\; \text{request}\bigr)$$
    No $\text{protocol}$ field is present, as no protocol was activated.
    \item \textbf{Protocol-level SDR} --- emitted when the LCDM or LCRM emits $\bot$ for protocol $p_k$:
    $$\text{SDR}^{(k)} \;=\; \bigl(\text{protocol}^{(k)},\; \text{source}^{(k)},\; \text{request}^{(k)}\bigr)$$
    where:
    \begin{itemize}
        \item $\text{protocol}^{(k)} = p_k$: the activated protocol for which a failure occurred.
        \item $\text{source}^{(k)} \in \{\text{LCDM},\, \text{LCRM}\}$: the module that emitted $\bot$.
        \item $\text{request}^{(k)}$: a targeted description of the additional data required to resume the pipeline for $p_k$. The content of $\text{request}^{(k)}$ is determined by $\text{source}^{(k)}$ and the activation set $\hat{P}$; it is not a generic message.
    \end{itemize}
\end{enumerate}
The SDR does not terminate the pipeline for other activated protocols. Protocols without failures continue to produce text and structured bypass output independently.
\end{definition}

\begin{definition}[Standardized Intermediate Payload]
\label{def:lcwm_sdr_sip-sip}
The Standardized Intermediate Payload (SIP) is the output of the LCWM, comprising the two components of Definition 33:
$$\text{SIP} \;=\; \bigl(\,\text{text\_output},\;\text{structured\_bypass}\,\bigr)$$
The SIP constitutes the architectural boundary between the LCA framework and downstream systems. Its internal structure is not specified at the framework level; it is the subject of subsequent interoperability work (HL7 FHIR mapping). Downstream systems consume the SIP without interacting with any internal LCA module.
\end{definition}

\begin{definition}[LCWM Function]
\label{def:lcwm_sdr_sip-func}
The LCWM function is:
$$f_{LCW}\!\Bigl(\hat{P},\; \{\tilde{y}_{LCD}^{(k)}\}_{p_k \in \hat{P}_{LCD}},\; \{\tilde{y}_{LCR}^{(k)}\}_{p_k \in \hat{P}_{LCR}},\; \mathcal{D}_{out}\Bigr) \;=\; \text{SIP}$$
Unlike $f_{LCD}^{(k)}$ and $f_{LCR}^{(k)}$, the LCWM is not parametrized by an interchangeable implementation under the interface-contract substitution of Definition~\ref{def:lcdm_lcrm-mod_inter} (restricted to $\tau \in \{\text{LCD}, \text{LCR}\}$): $f_{LCW}$ is $\theta$-free, on par with $f_{DP}$ and $f_{CS}$ (Corollary~\ref{cor:lca-unchange}). A change to its internal NLG realization constitutes a new deployment of $f_{LCW}$, not a substitution within a declared interface contract.
\end{definition}

\subsection*{Properties of $f_{LCW}$}
\label{app:subsec:lcwm_sdr_sip-pptes}
\begin{enumerate}[\bf (P5.1)]
    \item \textbf{Exclusive external interface.} The LCWM is the sole module whose output exits the LCA framework. All other modules communicate exclusively with adjacent modules.
    \item \textbf{Coverage.} For every $p_k \in \hat{P}$, the text output of the LCWM contains at least one entry --- either narrative content or a Supplementary Data Request. No activated protocol produces a silent empty output.
    \item \textbf{SDR specificity.} Each $\text{SDR}^{(k)}$ identifies the specific module that failed and the specific data deficiency. Generic or uninformative requests are not admissible.
    \item \textbf{SIP as architectural boundary.} The SIP decouples the LCA core from volatile external IT infrastructure. Changes to EMR systems, FHIR versions, or interoperability standards do not require modification of any module upstream of the LCWM.
\end{enumerate}

\section{SIP Design Principles}
\label{app:sip_sdr_design}

\begin{enumerate}[\bf D1]
    \item \textbf{SIP as the LCWM's sole output type.} The LCWM \textbf{always} produces a SIP (Def. \ref{def:lcwm_sdr_sip-sip}), regardless of the situation. The top-level JSON is invariantly a SIP. When a module emits $\bot$, the resulting SDR is (i) a component of \texttt{lcwm\_narrative} (Def. \ref{def:lcwm_sdr_sip-comp} ; text (i)) and (ii) a structured object in \texttt{protocol\_outputs[k].sdr} --- or in \texttt{csm\_sdr} if the CSM itself has failed. There is no \texttt{payload\_type: "SDR"} as an alternative to \texttt{payload\_type: "SIP"}; this pattern conflated the type of the LCWM output with its content.

    \item \textbf{$D_{out}$: input references rather than content.} Def. \ref{def:lcwm_sdr_sip-sip} requires $D_{out}$ in the bypass but explicitly states: \textit{"The framework is agnostic as to persistence: $D_{out}$ may equal the original history when the DPM acts as the identity (P5), or a preprocessed representation otherwise; which representation(s) a downstream system archives is an implementation decision outside the framework."}
    
    $D_{out}$ contains normalized CT volumes, clinical note tokens --- including them in full within the JSON introduces unacceptable latency and redundancy: the source data already exists in the upstream Healthcare system (FHIR/PACS/EMR). Downstream, a FHIR consumer does not consume $D_{out}$; it retrieves the originals via their identifiers.
    
    \textbf{Adopted solution}: \texttt{input\_provenance.entry\_refs} logs for each entry $j$ its acquisition index, external identifier in the source system, and medical signature $\sigma_j$. The DPM profile is logged separately to enable $D_{out}$ reconstruction if needed (auditability). The medical signature of input entries is essential: it informs the FHIR translation about the resource types to reference (\texttt{ImagingStudy} vs \texttt{DocumentReference} vs \texttt{Observation}).
    
    This is not a violation of theory --- it is the explicit instantiation of the persistence choice left open by Def. \ref{def:lcwm_sdr_sip-sip}, consistent with the constraints of a FHIR environment.

    \item \textbf{CSM trace: V1 and V2 explicitly distinguished.} The CSM has two variants (Def. \ref{def:csm-csm-V2-csm_func}). The JSON trace distinguishes which one is active:
    \begin{itemize}
        \item \textbf{V1}: \texttt{v1\_param\_P\_lambda} is the a priori declared parameter (Def. \ref{def:csm-csm-V1-p_priori})
        \item \textbf{V2}: \texttt{v2\_per\_protocol\_scores} = the scores $h(e)_k$ per protocol and \texttt{v2\_threshold\_tau} = $\tau$ (Def. \ref{def:csm-csm-V2-act_func})
        \item \textbf{In both cases}: \texttt{activation\_set\_P\_hat} is the \textbf{output} $\hat{P}$ of the CSM --- this is what downstream modules consume.
    \end{itemize}
    Distinguishing $P_\lambda$ (V1 input) from $\hat{P}$ (CSM output) is conceptually important. In V1, they are equal by construction, but $P_\lambda$ is the a priori clinical declaration, while $\hat{P}$ is the effective routing decision.

    \item \textbf{Naming aligned with theory.} See structural correspondences compiled in Table \ref{tab:naming_alignment}.

\begin{table*}[width=.9\textwidth, cols=3, pos=tbp]
    \caption{Naming Aligned with Theory}
    \label{tab:naming_alignment}
    \begin{tabular*}{\tblwidth}{@{} LLL @{}}
        \toprule
        JSON Field & Theoretical Correspondence & Definition \\
        \midrule
        \texttt{activation\_set\_P\_hat} & $\hat{P} \subseteq \mathcal{P}$ & Def. \ref{def:csm-act_set}, \ref{def:csm-csm-V2-csm_func} \\
        \texttt{v1\_param\_P\_lambda} & $P_\lambda$ & Def. \ref{def:csm-csm-V1-p_priori} \\
        \texttt{v2\_threshold\_tau} & $\tau$ & Def. \ref{def:csm-csm-V2-act_func} \\
        \texttt{v2\_per\_protocol\_scores} & $h(e)_k$ & Def. \ref{def:csm-csm-V2-act_func} \\
        \texttt{sigma} & $\sigma = (\rho, u, \kappa)$ & Def. \ref{def:entry-med_sig} \\
        \texttt{y\_LCD} / \texttt{y\_LCR} & $\tilde{y}_{LCD}^{(k)}$ / $\tilde{y}_{LCR}^{(k)}$ & Def. \ref{def:lcdm_lcrm-lcdm-out}, \ref{def:lcdm_lcrm-lcrm-func} \\
        \texttt{support} & $\text{support}(p_k)$ & Def. \ref{def:lcdm_lcrm-catalog-prot_supp} \\
        \texttt{entry\_refs[j].j} & acquisition index $j$ & Def. \ref{def:entry-clinic_hist} \\
        \texttt{T} & history cardinality & Def. \ref{def:entry-clinic_hist}, \ref{def:entry-lca_input} \\
        \texttt{lifting\_note} & canonical lifting $\iota$ & Def. \ref{def:lcdm_lcrm-lcrm-func}, Rmk. \ref{rmk:lcdm_lcrm-lcrm-cond_act} \\
        \bottomrule
    \end{tabular*}
\end{table*}

    \item \textbf{\texttt{output\_type}: alignment with type($p_k$) from Def. \ref{def:lcdm_lcrm-out_space}.} Exactly three output types (Def. \ref{def:lcdm_lcrm-out_space}) are mapped in Table \ref{tab:output_types}.

\begin{table}[tbp]
\caption{Alignment with type($p_k$)}
\label{tab:output_types}
\begin{tabular*}{\tblwidth}{@{} LLL @{}}
\toprule
JSON Value & Correspondence & Components \\
\midrule
\texttt{"structural"} & $\mathcal{Y}_{struct}$ & \texttt{structural} N.null \\
$\quad$ & $\quad$ & \texttt{decisional} null \\
\texttt{"decisional"} & $\mathcal{Y}_{dec}$ & \texttt{structural} null \\
$\quad$ & $\quad$ & \texttt{decisional} N.null \\
\texttt{"structural\_x\_decisional"} & $\mathcal{Y}_{struct} \times \mathcal{Y}_{dec}$ & both N.null \\
\bottomrule
\end{tabular*}
\end{table}

    The decisional output is \textbf{always} a distribution $\Delta(\mathcal{L}^{(k)})$ with a \texttt{class\_distribution\_delta\_L\_k} object. The argmax is a separate derived field --- never the primary value. A string like \texttt{"NSCLC - Suspicion"} in place of a distribution erases the formal structure and is inadmissible.

    \item \textbf{Exact cataloged values for $\sigma$.} The values of $\sigma = (\rho, u, \kappa)$ are constrained by Defs. \ref{def:entry-prov}–\ref{def:entry-epist_cert}:
    \begin{itemize}
        \item $\rho \in \mathsf{Prov}$: \texttt{imaging/CT}, \texttt{imaging/MRI}, \texttt{imaging/PET}, \texttt{pathology/WSI}, \texttt{biology/panel}, \texttt{documentation/note}, \texttt{inference/AI}
        \item $u \subseteq \mathsf{Usg}$ (non-empty): \texttt{observation}, \texttt{diagnosis}, \texttt{procedure}, \texttt{medication}
        \item $\kappa \in \mathsf{Cert}$: \texttt{confirmed}, \texttt{suspected}, \texttt{inferred}
    \end{itemize}

    \item \textbf{SDR: mandatory source, targeted request (P5.3).} Def. \ref{def:lcwm_sdr_sip-sdr} defines two SDR variants:
    \begin{itemize}
        \item \textbf{Top-level SDR} (\texttt{csm\_sdr}): \texttt{source = "CSM"}, no \texttt{protocol\_id} field. Emitted when the CSM itself emits $\bot$.
        \item \textbf{Protocol-level SDR} (\texttt{protocol\_outputs[k].sdr}): \texttt{source $\in$ \{LCDM, LCRM\}}, \texttt{protocol\_id} present. Emitted when a module emits $\bot$ for $p_k$.
    \end{itemize}
    Property P5.3 requires that each SDR identifies the failed module and the specific deficiency. Generic requests are inadmissible. \texttt{source} is always present in both variants.

    \item \textbf{Algorithmic impermeability: model identities outside SIP.} Property P3.1: implementations satisfying the same interface contract are interchangeable. Exposing \texttt{"monai\_unet\_v1\_lung"} in the SIP creates a downstream dependency forbidden by theory. Model identities belong in the internal audit log, not in the SIP.

    \item \textbf{LCR-only case: lifting $\iota$ and medical signature of the source.} When $\text{support}(p_k) = \{\text{LCR}\}$, the lifting $\iota$ maps the confirmed historical entry to the LCDM-output type (Def. \ref{def:lcdm_lcrm-lcrm-func}, Rmk. \ref{rmk:lcdm_lcrm-lcrm-cond_act}). Theory (\ref{def:lcdm_lcrm-lcdm-out}) states $\rho_{LCD} = \text{inference/AI} $ by default, but this default assumes that the AI module produced the entry. Lifting bypasses the AI module by definition: the medical signature of the lifted entry \textbf{inherits from the confirmed source}, not from inference/AI.
    
    D9 implements the lifting exception already stated in Def. \ref{def:lcdm_lcrm-lcdm-out} and Def. \ref{def:lcdm_lcrm-lcrm-func}: in the canonical lifting case, $\rho_{LCD}$ is inherited from the confirmed source entry rather than assigned as \texttt{inference/AI}. The JSON makes this explicit: \texttt{status: "LIFTED"} + \texttt{sigma\_LCD} inherited from the source + explicit \texttt{lifting\_note} field. This pattern is FHIR-aligned (the source \texttt{DiagnosticReport} resource can be referenced directly).

    \item \textbf{CSM SDR vs protocol SDR: two structural levels.}
    \begin{itemize}
        \item \textbf{CSM failure}: \texttt{protocol\_outputs} is empty (no module was invoked), top-level \texttt{csm\_sdr} field is present.
        \item \textbf{Module failure (LCDM or LCRM)}: the SDR is in \texttt{protocol\_outputs[k].sdr}. Other protocols continue normally (Def. \ref{def:lcwm_sdr_sip-sdr}: \textit{"The SDR does not terminate the pipeline for other activated protocols."})
        \item \texttt{lcrm\_output: null} (absent from support) \\
        $\neq$ \text{lcrm\_output.status: "HALTED"} \\
        (invoked but failed). The distinction is semantically important for FHIR translation.
    \end{itemize}

    \item \textbf{Output lineage: reference from output to $D_{out}$ input.} A structural output component can carry a lineage reference to an entry in \texttt{input\_provenance.entry\_refs} via its acquisition index $j$. This reference identifies the entry that serves as the baseline for the derived component (e.g., temporal derivative relative to the baseline input $j$).
    
    This mechanism is symmetric to D2: pointer by acquisition index, never content inclusion. The lineage reference is not a component of $\text{type}(p_k)$ and does not appear in the declared type of the protocol---it is a provenance metadata of the output, attached alongside the typed signal.

    \item \textbf{$\bot$ vs CODE\_FAIL.} $\bot$ is reserved exclusively for inputs in $\mathcal{D}_\bot^{(k)}$ (Def. \ref{def:lcdm_lcrm-mod_inter}). An inferential failure on a valid input (outside $\mathcal{D}_\bot^{(k)}$) does not produce $\bot$ and does not generate an SDR. It returns:
    \begin{itemize}
        \item \texttt{status: "SUCCESS"}
        \item \texttt{class\_distribution\_delta\_L\_k} with total mass on \texttt{CODE\_FAIL}$^{(k)}$ --- class declared a priori in\\
        \texttt{class\_catalog\_L\_k}
        \item \texttt{sdr: null}
    \end{itemize}
    A downstream consumer distinguishes \texttt{CODE\_FAIL} from a valid output by the membership of \texttt{argmax\_class} in the failure codes declared in the specification of protocol $p_k$.
\end{enumerate}

\begin{figure*}[htbp]
\begin{multicols}{2}
\begin{lstlisting}[language=json]
{
  // --- Header ---
  "sip_version": "<string>",
  "run_id": "<uuid>",
  "timestamp_utc": "<ISO8601>",

  // --- Input provenance (D_out by reference, see D2) ---
  "input_provenance": {
    "patient_ref": "<string>",
    "T": "<int >= 1>",
    "dpm_profile_id": "<string>",
    "entry_refs": [
      {
        "j": "<int>",
        "external_id": "<string>",
        "sigma": {
          "rho": "<Prov value>",
          "u":   ["<Usg value>"],
          "kappa": "<Cert value>"
        }
      }
    ]
  },

  // --- CSM execution (Def. 23) ---
  "csm_execution": {
    "variant": "V1" | "V2",
    "activation_set_P_hat": ["<protocol_id>"] | null,
    "status": "ACTIVATED" | "HALTED",
    "v1_param_P_lambda": ["<protocol_id>"] | null,
    "v2_per_protocol_scores": { "<protocol_id>": "<float [0,1]>" } | null,
    "v2_threshold_tau": "<float (0,1)>" | null
  },

  // --- CSM-level SDR (Def. 34) ---
  "csm_sdr": {
    "source": "CSM",
    "request": {
      "description": "<string>",
      "requested_data": { /* targeted free fields */ }
    }
  } | null,

  // --- Outputs per activated protocol ---
  "protocol_outputs": [
    {
      "protocol_id": "<string>",
      "support": ["LCD", "LCR"] | ["LCD"] | ["LCR"],

      "lcdm_output": {
        "status": "SUCCESS" | "HALTED" | "LIFTED",
        "characterized_entry": {
          "y_LCD": {
            "output_type": "structural" | "decisional" | "structural_x_decisional",
            "structural": {
              "domain_type": "<string>",
              "content_ref": "<uri>" | null,
              "inline_scalar_features": { /* named fields */ } | null,
              "temporal_derivative": {
                "baseline_entry_ref_j": "<int>",
                "delta_fields": { /* free fields */ }
              } | null
            } | null,
            "decisional": {
              "class_catalog_L_k": ["<label>"],
              "class_distribution_delta_L_k": { "<label>": "<float>" },
              "argmax_class": "<label>",
              "argmax_code_hint": "<string>" | null
            } | null
          },
          "sigma_LCD": {
            "rho": "<Prov value>",
            "u":   ["<Usg value>"],
            "kappa": "<Cert value>"
          },
          "lifting_note": "<string>" | null
        } | null,
        "interpretability": {
          "grad_cam_ref": "<uri>" | null,
          "attention_map_ref": "<uri>" | null
        } | null
      } | null,

      "lcrm_output": {
        "status": "SUCCESS" | "HALTED",
        "characterized_entry": {
          "y_LCR": {
            "output_type": "structural" | "decisional" | "structural_x_decisional",
            "structural": { /* same structure as y_LCD.structural */ } | null,
            "decisional": { /* same structure as y_LCD.decisional */ } | null
          },
          "sigma_LCR": {
            "rho": "inference/AI",
            "u":   ["<Usg value>"],
            "kappa": "<Cert value>"
          }
        } | null
      } | null,

      "sdr": {
        "protocol_id": "<string>",
        "source": "LCDM" | "LCRM",
        "request": {
          "description": "<string>",
          "requested_data": { /* targeted free fields (P5.3) */ }
        }
      } | null
    }
  ],

  // --- LCWM narrative (Def. 33(i)) ---
  "lcwm_narrative": "<string>"
}
\end{lstlisting}
\end{multicols}
\caption{System JSON Output Representation}
\label{fig:json_output}
\end{figure*}

\section{SIP Cases}
\label{app:sip_sdr_cases}
The concrete instantiations of these validation rules across the eight canonical execution paths (Case 1 to Case 8) are summarized in Table \ref{tab:covered_cases} ; files can be found in the code repository \url{<...>} ; scenarios are detailed below.

\begin{table*}[width=.9\textwidth, cols=6, pos=tbp]
\caption{Summary of Covered Cases (T1 Source Mapping)}
\label{tab:covered_cases}
\begin{tabular*}{\tblwidth}{@{} LLLLLL @{}}
\toprule
Case & CSM Variant & Support & LCDM & LCRM & SDR \\
\midrule
1 & V1 & \{LCD, LCR\} & SUCCESS & SUCCESS & --- \\
2 & V2 multi-proto & \{LCD, LCR\} + \{LCD\} & SUCCESS & SUCCESS / null & --- \\
3 & V2 & --- & not invoked & not invoked & CSM top-level \\
4 & V1 & \{LCD, LCR\} & HALTED & null (not invoked) & LCDM level \\
5 & V1 & \{LCD, LCR\} & SUCCESS & HALTED & LCRM level \\
6 & V1 & \{LCR\} & LIFTED ($\iota$) & SUCCESS & --- \\
7 & V1 & \{LCD, LCR\} $\times$ 2 proto & SUCCESS / HALTED & SUCCESS / null & LCDM level (proto 2) \\
8 & V1 & \{LCD, LCR\} & SUCCESS (CODE\_FAIL) & HALTED & LCRM level \\
\bottomrule
\end{tabular*}
\end{table*}

\begin{enumerate}[\bf \text{Case} 1]
    \item Lung, CT + clinical note, comparison with baseline ($T=3$). V1 activated by the clinician. LCDM and LCRM succeed. Output type = structural $\times$ decisional for LCDM, decisional for LCRM.
    \item Thoracic CT with mediastinal adenopathies. V2 activates two protocols above $\tau=0.70$. \\
    \texttt{lymphoma\_pipeline} is support=\{LCD\} (diagnosis only).
    \item Insufficient data (clinical note only), no V2 protocol exceeds $\tau$. CSM emits $\bot$. No downstream module invoked.
    \item V1 activated, CSM succeeds. \\
    LCDM breast\_pipeline emits $\bot$: $T=1$, no prior MRI baseline to calculate the temporal derivatives required by the protocol. LCRM not invoked (its precondition --- valid LCDM output --- is not met).\\
    The distinction \texttt{lcrm\_output: null} (not invoked) vs \texttt{lcrm\_output.status: "HALTED"} (invoked but failed) is semantically important for FHIR: in the first case, no LCRM SDR Task is created.
    \item LCDM succeeds (NSCLC, $P=0.61$, $\kappa=\text{inferred}$). LCRM emits $\bot$: the synthesis policy \texttt{synth\_k} of the \texttt{lung\_rads\_pipeline} protocol only allows therapeutic routing if $\kappa \geq \text{suspected}$; the LCDM entry with $\kappa=\text{inferred}$ is insufficient. The LCDM output remains valid in the bypass.
    \item NSCLC confirmed by prior biopsy ($j=2$, pathology/WSI, $\kappa=\text{confirmed}$). The clinician activates \texttt{nsclc\_treatment\_pipeline} (support=\{LCR\}). The LCDM is bypassed; the lifting $\iota$ maps entry $j=2$ to the LCDM-output type (point-mass distribution, $\sigma$ inherited from $j=2$, see D9). The LCRM produces a first-line recommendation.
    \item V1 activates two protocols.\\
    \texttt{lung\_rads\_pipeline}: LCDM SUCCESS + LCRM SUCCESS. \texttt{lymphoma\_pipeline}: LCDM HALTED (required mediastinal CT absent, $T=1$, only documentation/note available --- input in $\mathcal{D}_\bot^{(\text{lymphoma})}$), LCRM not invoked. Both protocols coexist \\
    in \texttt{protocol\_outputs}. Demonstrates Def.~\ref{def:lcwm_sdr_sip-sdr} non-termination clause and P3.2/P4.2.
    \item V1, single protocol. The LCDM receives a valid input (not in $\mathcal{D}_\bot^{(k)}$): the CT is present and well-formed. Internal inference fails (e.g., model convergence failure). Per D12 and Def.~\ref{def:lcdm_lcrm-mod_inter} remark: no $\bot$, no SDR. The LCDM returns \texttt{status: "SUCCESS"} with full probability mass on \texttt{CODE\_FAIL} --- a class declared \emph{a priori} in \texttt{class\_catalog\_L\_k}. The LCRM precondition (requires $\kappa \in \{\text{confirmed}, \text{suspected}\}$) is not met by a \texttt{CODE\_FAIL} output ($\kappa = \text{inferred}$): LCRM emits $\bot$ and a protocol-level SDR is generated.
\end{enumerate}

\section{Proof-of-Concept Protocol and Detailed Results}
\label{app:poc}

\paragraph{Environment.}
All scenarios use the deterministic variant V1 on a single lung protocol. The diagnostic module is instantiated with declared reference stubs conforming to $\mathcal{I}_{LCD}^{(\mathrm{lung})}$; for S2 these are two independently labeled stub instances (Stub A, Stub B) emitting distinct simulated diagnostic content, without executing any underlying neural architecture. The remedy module follows the same design, declared as a rule-based stub (Lung-RADS-style criteria). Inputs are synthetic entries tagged with lung CT provenance ($\rho = \text{imaging/CT}$); the mock protocol $p_{\mathrm{mock}}$ in S4 is a declared baseline diagnostic stub. The configuration isolates orchestration behaviour and reports no classifier-level metric.

\paragraph{Metric definitions.}
\begin{itemize}
  \item \textbf{S1.} \emph{completion rate}: fraction of runs reaching the LCWM without $\bot$; \emph{SIP schema valid}: payload conforms to the schema of \ref{app:sip_sdr_design}; \emph{$\sigma_{LCD}$ correct}: emitted signature equals the protocol-declared diagnostic signature.
  \item \textbf{S2.} \emph{$\pi$-equality}: Boolean equality of the orchestration-structural projection (activation set, routing, $\bot$-set, schema) across the model pair; \emph{content difference}: $\tilde{y}_{LCD}^{A}\neq\tilde{y}_{LCD}^{B}$.
  \item \textbf{S3 (module-level).} $\text{SDR recall} \\=(\bot\ \text{raised})\wedge(\text{source correct})\wedge   (\text{SDR emitted})\wedge(\neg\,\text{generic request})$. Both the corrupted-CT and missing-modality modes raise $\bot$ at the diagnostic module, producing a protocol-level SDR with $\text{source}=\textsf{LCDM}$ (consistent with Q1: under V1 the CSM is content-blind and never originates $\bot$).
  \item \textbf{S3 (type invariant).} For an empty history ($T=0$) the input violates $T\geq 1$ (Def. \ref{def:entry-lca_input}) and is rejected at construction before any module runs. Recorded values: $\textsf{bottom\_raised}=\text{true}$, $\textsf{bottom\_source\_correct}=\text{null}$ (no module reached \textsf{HALTED}), $\textsf{sdr\_emitted}=\text{false}$ (a construction error is not a Def.~\ref{def:lcwm_sdr_sip-sdr} SDR), $\textsf{sdr\_request\_is\_generic}=\text{null}$. This is a type-safety property, reported separately from SDR recall.
  \item \textbf{S4.} \emph{branch independence}: each activated protocol produces its outputs without reference to the other; \emph{composite schema valid}: the multi-protocol payload conforms to the schema; \emph{cardinality}: $|\textsf{protocol\_outputs}|=K=2$.
\end{itemize}

\paragraph{Sample sizes and latency.}
$N=10$ for every scenario and failure mode (S2: $10$ model pairs). Latency is measured per module as wall-clock mean $\pm$ std over the $N=10$ S1 runs (Table~\ref{tab:poc-s1}); the orchestration-only cost (DPM, CSM, LCWM) is $\approx 0.04$~ms. The sub-millisecond total is negligible relative to production inference times ($10^{2}$--$10^{3}$~ms), which the lightweight reference components do not reproduce.

\paragraph{Multi-protocol latency projection.} For $K$ activated protocols, the reference implementation of the PoC executes the LCDM/LCRM branches sequentially -- an engineering choice of the current codebase, not a structural requirement of the framework:
\begin{align}
T_{LCA}^{seq} &= T_{DPM} + T_{CSM} + \notag \\
&\sum_{k=1}^{K}\left(T_{LCDM}^{(k)} + T_{LCRM}^{(k)}\right) + T_{LCWM}
\end{align}
Because branch independence is established both formally (P3.2, P4.2) and empirically (S4, Table~\ref{tab:poc}: 100\% branch independence for $K=2$), concurrent execution across the $K$ branches is architecturally admissible and yields the lower bound:
\begin{align}
T_{LCA}^{par} &= T_{DPM} + T_{CSM} + \notag \\
& \max_{k=1,\dots,K}\left(T_{LCDM}^{(k)} + T_{LCRM}^{(k)}\right) + T_{LCWM}
\end{align}
The present PoC reports $T_{LCA}^{seq}$ as measured for the single-protocol case (Table~\ref{tab:poc-s1}); benchmarking $T_{LCA}^{par}$ under concurrent execution across multiple activated protocols is left to future implementation work.

\printcredits

\section*{Declarations}
\label{sec:declarations}

\subsection*{Ethics approval and consent to participate}
\label{subsec:declarations-ethical}
This study used no real patient data. All Proof-of-Concept inputs were synthetically generated to exercise the orchestration pipeline's structural and control-flow properties. No human subject data was collected, and no interventions were performed on human participants. Therefore, institutional review board (IRB) approval and explicit patient informed consent were not required for this theoretical and computational framework validation.

\subsection*{Data availability}
\label{subsec:declarations-data_avail}
No external datasets were used in this study. All Proof-of-Concept inputs were synthetically generated; the generation logic is included in the code repository referenced below.

\subsection*{Code availability}
\label{subsec:declarations-code_avail}
To ensure complete reproducibility and support the algorithmic impermeability claim, the overarching orchestration code, the declared stubs, and the finalized Standardized Intermediate Payload (SIP) and Supplementary Data Request (SDR) JSON schemas used during the Proof of Concept are publicly accessible. The repository can be found at: \url{https://github.com/MARRAKCHIGhassen/lca}.

\subsection*{Declaration of Competing Interest}
\label{subsec:declarations-coi}
The authors declare that they have no known competing financial interests or personal relationships that could have appeared to influence the work reported in this paper.

\subsection*{Funding}
\label{subsec:declarations-fund}
This research did not receive any specific grant from funding agencies in the public, commercial, or not-for-profit sectors.

\subsection*{Declaration of generative AI and AI-assisted technologies in the writing process}
\label{subsec:declarations-gen_ai}
During the preparation of this work the authors used Gemini By Google and Claude by Anthropic, in order to improve language readability, structure, and academic translation. After using this tool/service, the authors reviewed and edited the content as needed and take full responsibility for the content of the publication.

\bibliographystyle{cas-model2-names}

\bibliography{bibliography}

@article{GD_Net,
  author    = {Lin, J. and Deng, W. and Wei, J. and Zheng, J. and Chen, K. and Chai, H. and Zeng, T. and Tang, H.},
  title     = {GD-Net: An Integrated Multimodal Information Model Based on Deep Learning for Cancer Outcome Prediction and Informative Feature Selection},
  journal   = {Journal of Cellular and Molecular Medicine},
  year      = {2024},
  volume    = {28},
  number    = {23},
  pages     = {e70221},
  doi       = {10.1111/jcmm.70221},
  pmid      = {39628446},
  pmcid     = {PMC11615516},
  url       = {https://doi.org/10.1111/jcmm.70221}
}

@article{DeepOmix,
  author    = {Zhang, B. and Wan, Z. and Luo, Y. and Zhao, X. and Samayoa, J. and Zhao, W. and Wu, S.},
  title     = {Multimodal integration strategies for clinical application in oncology},
  journal   = {Frontiers in Pharmacology},
  year      = {2025},
  volume    = {16},
  pages     = {1609079},
  doi       = {10.3389/fphar.2025.1609079},
  url       = {https://doi.org/10.3389/fphar.2025.1609079}
}

@article{DeepClinMed_PGM,
  author    = {Wang, Z. and Lin, R. and Li, Y. and Zeng, J. and Chen, Y. and Ouyang, W. and Li, H. and Jia, X. and Lai, Z. and Yu, Y. and Yao, H. and Su, W.},
  title     = {Deep learning-based multi-modal data integration enhancing breast cancer disease-free survival prediction},
  journal   = {Precision Clinical Medicine},
  year      = {2024},
  volume    = {7},
  number    = {2},
  pages     = {pbae012},
  doi       = {10.1093/pcmedi/pbae012},
  pmid      = {38912415},
  pmcid     = {PMC11190375},
  url       = {https://doi.org/10.1093/pcmedi/pbae012}
}

@article{DeepMultimodal_Colorectal,
  author    = {Wang, N. and Lin, J. and Li, W. and Lyu, Y. and Jiang, Y. and Ni, Z. and Huang, Q. and Chen, H. and Yan, Q. and Huang, C.},
  title     = {Deep multimodal state-space fusion of endoscopic-radiomic and clinical data for survival prediction in colorectal cancer},
  journal   = {npj Digital Medicine},
  year      = {2025},
  volume    = {8},
  number    = {1},
  pages     = {801},
  doi       = {10.1038/s41746-025-02236-3},
  url       = {https://doi.org/10.1038/s41746-025-02236-3}
}

@article{Multimodal_Review_2024,
  author    = {Waqas, A. and Tripathi, A. and Ramachandran, R. P. and Stewart, P. A. and Rasool, G.},
  title     = {Multimodal data integration for oncology in the era of deep neural networks: a review},
  journal   = {Frontiers in Artificial Intelligence},
  year      = {2024},
  volume    = {7},
  pages     = {1408843},
  doi       = {10.3389/frai.2024.1408843},
  url       = {https://doi.org/10.3389/frai.2024.1408843}
}

@article{Precision_Oncology_Survey,
  author    = {Yang, H. and Yang, M. and Chen, J. and Yao, G. and Zou, Q. and Jia, L.},
  title     = {Multimodal deep learning approaches for precision oncology: a comprehensive review},
  journal   = {Briefings in Bioinformatics},
  year      = {2025},
  volume    = {26},
  number    = {1},
  pages     = {bbae699},
  doi       = {10.1093/bib/bbae699},
  url       = {https://doi.org/10.1093/bib/bbae699}
}

@article{MUSK_Stanford,
  author    = {Xiang, J. and Wang, X. and Zhang, X. and Xi, Y. and Eweje, F. and Chen, Y. and Li, Y. and Bergstrom, C. and Gopaulchan, M. and Kim, T. and Yu, K. and Willens, S. and Olguin, F. M. and Nirschl, J. J. and Neal, J. and Diehn, M. and Yang, S. and Li, R.},
  title     = {A vision--language foundation model for precision oncology},
  journal   = {Nature},
  year      = {2025},
  volume    = {638},
  number    = {8051},
  pages     = {769--778},
  doi       = {10.1038/s41586-024-08378-w},
  url       = {https://doi.org/10.1038/s41586-024-08378-w}
}

@article{Multimodal_Radiotherapy_NSCLC,
  author    = {Niecikowski, A. and Gupta, S. and Suarez, G. and Kim, J. and Chen, H. and Guo, F. and Long, W. and Deng, J.},
  title     = {A Multi-Modal Deep Learning-Based Decision Support System for Individualized Radiotherapy of Non-Small Cell Lung Cancer},
  journal   = {International Journal of Radiation Oncology Biology Physics},
  year      = {2022},
  volume    = {114},
  number    = {3, Supplement},
  pages     = {e100--e101},
  doi       = {10.1016/j.ijrobp.2022.07.894},
  url       = {https://doi.org/10.1016/j.ijrobp.2022.07.894},
  note      = {ASTRO Annual 2022 Meeting}
}

@article{Lee2010,
  author    = {Lee, Jaehoon and Kim, JeongAh and Cho, Insook and Kim, Yoon},
  title     = {Integration of workflow and rule engines for clinical decision support services},
  journal   = {Studies in Health Technology and Informatics},
  volume    = {160},
  number    = {Pt 2},
  pages     = {811--815},
  year      = {2010},
  issn      = {0926-9630},
  publisher = {Netherlands},
  pmid      = {20841798},
  doi       = {},
  url       = {https://pubmed.ncbi.nlm.nih.gov/20841798/}
}

@article{Huser2011,
  author    = {Huser, Vojtech and Rasmussen, Luke V and Oberg, Ryan and Starren, Justin B},
  title     = {Implementation of workflow engine technology to deliver basic clinical decision support functionality},
  journal   = {BMC Medical Research Methodology},
  volume    = {11},
  pages     = {43},
  year      = {2011},
  month     = {Apr},
  issn      = {1471-2288},
  doi       = {10.1186/1471-2288-11-43},
  pmid      = {21477364},
  pmc       = {PMC3079703},
  url       = {https://pubmed.ncbi.nlm.nih.gov/21477364/}
}

@article{carbonaro2025raw,
  title={From raw data to research-ready: A FHIR-based transformation pipeline in a real-world oncology setting},
  author={Carbonaro, Antonella and Giorgetti, Luca and Ridolfi, Lorenzo and Pasolini, Roberto and Pagliarani, Andrea and Cavallucci, Martina and Andal{\`o}, Alice and Del Gaudio, Livia and De Angelis, Paolo and Vespignani, Roberto and others},
  journal={Computers in Biology and Medicine},
  volume={197},
  pages={111051},
  year={2025},
  publisher={Elsevier}
}

@article{Jung2022,
  author  = {Jung, S and Bae, S and Seong, D and Oh, O and Kim, Y and Yi, B},
  title   = {Shared Interoperable Clinical Decision Support Service for Drug-Allergy Interaction Checks: Implementation Study},
  journal = {JMIR Medical Informatics},
  year    = {2022},
  volume  = {10},
  number  = {11},
  pages   = {e40338},
  doi     = {10.2196/40338},
  url     = {https://medinform.jmir.org/2022/11/e40338}
}

@article{bayor2025designing,
  title={Designing clinical decision support systems (CDSS)—A user-centered lens of the design characteristics, challenges, and implications: Systematic review},
  author={Bayor, Andrew A and Li, Jane and Yang, Ian A and Varnfield, Marlien},
  journal={Journal of Medical Internet Research},
  volume={27},
  pages={e63733},
  year={2025},
  publisher={JMIR Publications Toronto, Canada}
}

@article{lopez2026empirical,
  title={An Empirical Analysis of Calibration and Selective Prediction in Multimodal Clinical Condition Classification},
  author={L{\'o}pez, L and Shamout, Farah E and Rudner, Tim GJ},
  journal={arXiv preprint arXiv:2603.02719},
  year={2026}
}



\end{document}